\documentclass[runningheads]{llncs}
\usepackage[T1]{fontenc}
\usepackage{graphicx}
\usepackage{booktabs}
\usepackage[misc]{ifsym}
\newcommand{\corr}{(\Letter)}
\usepackage{mwe}

\usepackage{mathtools}
\usepackage{caption}
\usepackage{subcaption}
\usepackage{float}
\usepackage{makecell} 
\usepackage{amstext}
\usepackage{amssymb}
\usepackage{graphicx}
\usepackage{algorithm}          
\usepackage{algorithmicx}
\usepackage[noend]{algpseudocode}

\usepackage{xcolor}
\usepackage{multirow}
\usepackage{hyperref}

\newcommand{\BibTeX}{\rm B\kern-.05em{\sc i\kern-.025em b}\kern-.08em\TeX}

\begin{document}

\title{Neurosymbolic Reasoning with \\ Incremental Knowledge for Sample Efficient \\
Hierarchical Reinforcement Learning}
\toctitle{Neurosymbolic Reasoning with Incremental Knowledge for Sample Efficient Hierarchical Reinforcement Learning} 


\titlerunning{Neurosymbolic Reasoning with InK for Sample Efficient HRL}

\author{Subrat Prasad Panda\inst{1,2} \corr \and
Blaise Genest\inst{2,3} \and
Arvind Easwaran\inst{1}}
\tocauthor{Subrat Prasad Panda, Blaise Genest, Arvind Easwaran} 
\authorrunning{S.P. Panda et al.}

\institute{CCDS, NTU Singapore \email{\{subratpr001@e.,arvinde@\}ntu.edu.sg}
\and
CNRS@CREATE, Singapore \email{blaise.genest@cnrsatcreate.sg}
\and
IPAL, CNRS, France}

\maketitle              

\begin{abstract}
(Flat) Reinforcement Learning (RL) agents face significant challenges in environments with sparse rewards that require long-horizon reasoning. A compelling approach to improve sample efficiency is to incorporate knowledge into learning and decision-making. In standard Hierarchical RL (HRL), knowledge is encoded in a fixed, non-updatable form, such as architectural choices, and remains unchanged throughout learning. With fixed HRL, reasoning with incremental knowledge learned during exploration is impractical before sufficient environmental knowledge is acquired, leading to poor sample efficiency.
In this work, we propose neurosymbolic HRL with {\em Incremental Knowledge (InK)}: symbolic high-level components perform {\em symbolic planning} (e.g. using $D^*$) on an updatable representation of current InK, while low-level goal-conditioned neural modules learn motion primitives through experience using reward shaping. Experiments on navigation tasks demonstrate that incorporating InK substantially improves sample efficiency.
Additionally, to perform {\em optimal} symbolic planning given {\em prior} knowledge about the world, we develop Belief World Tree Search.
The code is available at \url{https://github.com/CPS-research-group/ink_bwts}.

\end{abstract}

\section{Introduction}\label{intro}
Deep RL (DRL) has shown remarkable success in high-dimensional control problems, yet in domains requiring long-horizon planning and sparse rewards, purely end-to-end approaches suffer from prohibitive sample complexity \cite{zhao2024learning,goal-misgeneralization22}. Hierarchical RL (HRL) mitigates this by decomposing tasks across abstraction levels \cite{hrl_survey21}, with most approaches using neural network (NN) policies at both levels. More recent work explores neurosymbolic approaches \cite{acharya2023neurosymbolic}, where the high-level component is symbolic and the low-level component remains neural. In these approaches, the symbolic high-level maintains an abstract world model over which planning is performed, while the neural low-level executes the resulting subgoals through goal-conditioned policies. For instance, methods such as SoRB \cite{eysenbach2019sorb} and RGL \cite{rgl2024} first learn a full abstract graph over the state space through extensive exploration, after which planning is performed over the learned graph for deployment. This non-incremental knowledge (non-InK) paradigm assumes the world model must be fully constructed before any goal-directed behavior begins, requiring the agent to exhaustively explore the environment upfront, regardless of the actual goal, making these methods fundamentally inefficient.

In contrast, classical robotics planning often adopts incremental knowledge (InK), where the abstract world model is updated during execution, for example from execution traces \cite{ng2019incremental}.
A simple example is Active SLAM \cite{cadena2016slam}, where newly discovered obstacles trigger replanning as the map is incrementally updated. Algorithms such as $D^*$ (Dynamic $A^*$) replan paths using the current InK, incorporating newly detected obstacles \cite{stentz1995focussed,stentz1994optimal,koenig2003performance}.
The closest line of work to HRL arises in Task and Motion Planning (TAMP), where motion constraints discovered by the motion planner (low level) are incrementally incorporated into task planning (high level) \cite{tamper24,noseworthy2021active}. Similarly, \cite{lyu2019sdrl} learns a meta-policy to constrain plan feasibility when composing low-level policies. However, these approaches typically assume symbolic components at both high and low levels.

In neurosymbolic HRL, by contrast, the low-level controller is typically a neural policy with unknown dynamics, often pretrained independently of the target environment, e.g., a policy trained in an obstacle-free maze deployed in a maze with walls. This creates a semantic gap between the high-level abstract model and the low-level policy: what the high-level planner assumes feasible may be impossible to execute. So, high-level plans can become infeasible at execution due to incomplete knowledge, stochastic effects, or misdetected obstacles.
Such failures highlight the need for InK to update the abstract world model based on execution feedback. Incorporating InK better reflects human reasoning \cite{sharma2022map}, where agents act with partial knowledge, update it as new information is acquired, and refine their decisions accordingly, rather than exhaustively exploring the environment upfront.
However, existing non-incremental neurosymbolic HRL approaches are structurally unable to capture this policy-dependent world model update, and classical InK approaches assume symbolic components at both levels, leaving the use of InK in neurosymbolic HRL underexplored.

Furthermore, using InK in HRL enables incorporating structural knowledge into the high-level abstract world model. Since knowledge is partial at any time, planning must consider multiple possible belief worlds (a set of possible abstract world models consistent with the knowledge acquired thus far) to account for uncertainty \cite{sharma2022map}. While planning under uncertainty is often cast as a Partially Observable MDP (POMDP), where the transition model is known but states are hidden, InK differs fundamentally: states are fully observable, but the world dynamics are unknown and discovered incrementally. BAMDPs address unknown dynamics by reducing to a POMDP over augmented belief states; BAMCP~\cite{bamdp2013}, which adapts POMCP~\cite{pomcp_nips10} to this reduction, models uncertainty using Bayesian priors over transition probabilities but assumes independence across states.
However, some structural knowledge cannot be captured under independence assumptions: for instance, knowing that ``there is exactly one wall at an unknown location'' means that discovering a wall at one location immediately constrains all other locations.
As a result, these approaches cannot effectively exploit such prior structural constraints, and simple incremental replanning algorithms like $D^*$ do not readily incorporate this knowledge. Although these algorithms provide strong worst-case performance \cite{koenig2003performance}, they are not optimal when averaged over belief worlds.

In this work, we propose neurosymbolic HRL with InK to address these challenges. At the high level, a symbolic planner maintains an updatable abstract world model, replanning with $D^*$ (it could be any other incremental planner) as new information arrives from low-level execution. At the low level, a goal-conditioned neural policy handles continuous motion control and acts as a sensor: its successes and failures directly inform abstract world model updates. This tight coordination eliminates the costly upfront exploration phase of non-InK methods, enabling goal-directed behavior from the very first episode. To further exploit structural prior knowledge over possible worlds, we introduce Belief World Tree Search (BWTS), which adapts MCTS \cite{browne2012survey} to plan optimally on average over all worlds consistent with the structural priors, addressing the fundamental limitations of both incremental planners like $D^*$ and prior-based planners such as BAMCP in the InK setting.

\smallskip
\noindent \textbf{Our Main Contributions:}
\begin{itemize}
\item We propose a neurosymbolic HRL for continuous control problems, incorporating structural InK through an updatable symbolic representation.
\item We develop the BWTS algorithm converging towards an {\em optimal} policy for unknown environments among a set of possible worlds capturing {\em prior} structural knowledge.
\item We empirically demonstrate, using navigation task as a case study, that InK HRL outperforms non-InK HRL in sample efficiency.
\end{itemize}

\section{Related Work}\label{lit_review}

\noindent
\textbf{Hierarchical RL (HRL).}
End-to-end flat DRL has been developed for long-horizon reasoning 
using memory, auxiliary losses, or reward shaping \cite{mirowski2017learning,zhu2017target,zhao2024learning}. 
These techniques can become inefficient for complex tasks \cite{nachum2020why}.
HRL mitigates this by structuring policies across abstraction levels, improving exploration and scalability in sparse-reward settings \cite{hutsebaut2022hierarchical}. 
Foundational options and feudal frameworks introduce temporally extended actions \cite{sutton1999between}. Goal-conditioned HRL refines this by letting high-level policies set subgoals 
\cite{nasiriany2019planning,nachum2018data,levy2017learning,lei2025goalconditioned}. 
However, most HRL approaches lack explicit symbolic knowledge representations.

\smallskip
\noindent
\textbf{Graph Neurosymbolic HRL.} In graph-based methods such as SoRB \cite{eysenbach2019sorb}, the high-level is a graph on which planning is done and low-level is a NN. Similar hierarchical formulations can be found in landmark-guided HRL \cite{huang2019mapping,kim2021landmarkguided}, world graphs \cite{shang2020learning,savinov2018semiparametric},
abstract/refinement of world model with AI2 tool \cite{zadem2024reconciling},
state-space partitioning \cite{shah2024hierarchical},
probabilistic road maps \cite{prm_rl,cprm_rl}, PAHRL \cite{gieselmann2021planning}, RL-RRT \cite{rrt-rl}, and RGL \cite{rgl2024}.
It typically follows a non-InK paradigm in which a complete model must be acquired before reasoning. 

\smallskip
\noindent
\textbf{Incremental Knowledge.}
The overarching idea of InK appears in classical symbolic planning, where it is used to refine domain definitions or augment them with constraints derived from execution traces \cite{ng2019incremental,tamper24,noseworthy2021active}. In this work, we apply a similar idea to HRL to improve sample efficiency. In the RL setting, incremental learning \cite{khetarpal2022towards,meng2025preserving} has been used to learn parametric transition functions; however, incorporating explicit structural knowledge requires an updatable symbolic representation, which is the focus of this work. BAMCP \cite{bamdp2013} proposes to incorporate such knowledge incrementally in symbolic domain, but assumes independence between probability distributions of different state-action pairs: an assumption that does not hold in our setting, as we demonstrate experimentally.

\smallskip
\noindent
\textbf{Other Neurosymbolic Approaches.}
Neural and symbolic components can be combined in multiple ways. One common approach uses symbolic planning at a high level to guide exploration and direct low-level policy learning \cite{lyu2019sdrl,kokel2021deep,mayr2022combining} \cite{garrett2020pddlstream,guan2022leveraging,prakash2022towards}. Closely related, \cite{mao2023pdsketch} learns an abstract world model from demonstrations, but assumes a symbolic low-level policy (geometric motion planning). In contrast, our approach incrementally learns the abstract world model from real execution traces given a neural low-level policy, making the abstract world model policy-dependent and reducing the semantic gap between high-level abstractions and actual low-level behavior.

\smallskip
\smallskip
\noindent
\textbf{Monte Carlo Tree Search.}
Monte Carlo Tree Search (MCTS) is a dominant paradigm for {\em planning} in graphs, combining stochastic rollouts with tree expansion to approximate values without requiring a full transition model of the environment \cite{coulom2006efficient,kocsis2006bandit}, with success in complex domains, e.g. (2 player) games \cite{browne2012survey}. POMCP~\cite{pomcp_nips10} solves POMDPs where the transition model is known but states are partially observable, using particle filtering to maintain beliefs over hidden states. BAMCP~\cite{bamdp2013} extends this to unknown transitions by casting the problem as a BAMDP, which reduces to a POMDP over an augmented state space encoding beliefs over transition parameters, and then applies POMCP-style tree search; however, this reduction requires independence across state-action pairs, preventing it from capturing structural constraints (e.g., ``exactly one wall at an unknown location''). In contrast, BWTS addresses a fully observable but unknown setting: it maintains beliefs over structured sets of possible worlds without independence assumptions and uses strategic rather than probabilistic rollouts.

\section{Setting and Methodology}\label{methods}

\subsection{The Incremental Knowledge Setting}

We introduce first the (structural) {\em Incremental Knowledge (InK)} setting: the agent evolves in a world $W_{\text{true}}$ that it does not know. To model everything that it believes could happen, we consider a set $\mathcal{W}_0$ of possible worlds,
with $W_{\text{true}} \in \mathcal{W}_0$: the agent knows that the possible outcomes of an action $a$ after observations $\rho$ are exactly those compatible with some world $W \in \mathcal{W}_0$. This set $\mathcal{W}_0$ also serves to define the average reward over all the possible choices of $W \in \mathcal{W}_0$. Last, the choice of $\mathcal{W}_0$ encodes any (structural) {\em prior knowledge} on the world.

A Belief World Process is a tuple $\mathcal{M} = (\mathcal{W}_0, \mathcal{S},$  $\mathcal{A}, (\mathcal{T}_W)_{W \in \mathcal{W}_0}, C)$, where $\mathcal{W}_0$ is a finite set of all possible worlds, 
$\mathcal{S} \subseteq \mathbb{R}^n$ is the common {\em state space} of all the worlds, 
and $\mathcal{A}$ the common {\em action space} of all the worlds. 
The {\em transition dynamics} depends upon the worlds: 
given $W \in \mathcal{W}_0$, we have $\mathcal{T}_W: \mathcal{S} \times \mathcal{A} \times \mathcal{S} \rightarrow [0,1]$, 
denoting the probability to reach state $S'$ after trying action $a$ from state $S$ in world $W$.
It is deterministic if for all $S,a$, we have $\mathcal{T}_W(S,a,S')=1$ for a unique state $S'$, and 0 for other $S''$. 
The  {\em cost function} is $C : \mathcal{S} \times \mathcal{A} \rightarrow \mathbb{R}^{\geq 0}$, where the cost 
$C(S,a)=0$ only if $S$ is in the goal region $G \subset \mathcal{S}$, where $G$ is an attractor w.r.t. $\mathcal{T}_W$ for all $W$ (from $G$, all forthcoming costs are $0$). 
Some information on the real world $W_{\text{true}}$ can be extracted from the observed history $\rho=S_0, \ldots, S_{n-1}$: if this sequence has probability $0$ in a world $W$, then $W_{\text{true}} \neq W$.

A policy $\pi$ chooses action $a_n$ from history 
$\rho=S_0,a_1, \ldots, S_{n-1}$, that is $\pi(\rho)=a_n$. 
We denote by $C(\pi, W) \in \mathbb{R}^{\geq 0} \cup \{+\infty\}$ 
the cumulative cost to reach the goal in world $W \in \mathcal{W}_0$ when policy $\pi$ is used.
We then denote $C(\pi)$ the average of $C(\pi,W)$ over all the worlds 
$W \in \mathcal{W}_0$. We denote by $C^*$ the optimal cost in ${\mathcal W}_0$:
\begin{equation}
\label{eq1}
C^* = \inf_{\pi} C(\pi) = \inf_{\pi}\; \frac{\sum_{W\in \mathcal{W}_0}\!\left[ C(\pi,W) \right]}{|\mathcal{W}_0|}
\end{equation}
The {\em InK problem} consists in computing a strategy $\pi$ so as to minimize $C(\pi)$.

\subsection{A Neurosymbolic Hierarchical RL for InK}

Solving the RL problem end-to-end is particularly challenging in sparse-reward settings. We therefore adopt a hierarchical decomposition 
with symbolic planning directing continuous control.
At a high level, the state space $\mathcal{S}$ and action space $\mathcal{A}$ are abstracted into a finite set of symbolic (discrete) states $V$ (vertices in the discrete state graph) and a finite set of symbolic actions $A$, respectively.
Planning is performed in this abstracted space, where the high-level model maintains a world model $M$, which is an updatable representation that uses InK.
We execute the loop, as also illustrated in Fig.~\ref{fig:nesy_hrl}:
\begin{enumerate}
\item Symbolic planning is performed over $M$ using $D^*$ from current configuration $s$ to reach $G$ in an optimal way. Let $g$ be the first intermediate goal from the symbolic plan.
\item Direct the low-level to achieve the goal $g$.
\item Upon receiving new information from the low-level monitor, update (the InK in) $M$.
\end{enumerate}

\begin{figure}[t!]
  \centering
  \includegraphics[width=0.55\linewidth]{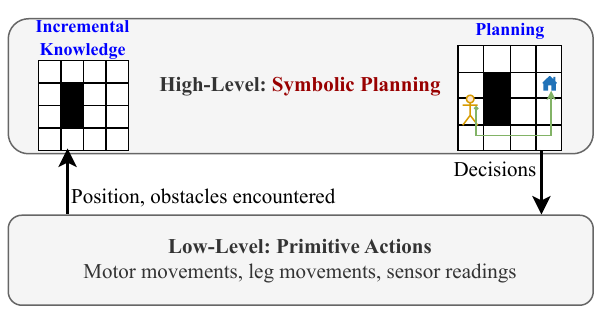}
  \caption{A Neurosymbolic HRL for InK}
  \label{fig:nesy_hrl}
\end{figure}

At the low level, a goal-conditioned policy $\pi_l$
considers the current state $s \in \mathcal{S}$ and the goal $g$ from the high level, and inductively executes a sequence $\rho$ with low-level actions $\pi_l(s,g)$ in the environment until $g$ is reached or fails. A monitor $mon$ monitors $\rho$ to detect environmental information. Upon returning, the low-level reports $s$ and $mon(\rho)$ to the high-level.

\paragraph{Coordination between levels.}
The high-level and low-level components interact through a tight coordination mechanism. The high-level maintains the symbolic representation $M$, for instance, a set of predicates, as well as the set of possible high-level abstract actions $A$. 
Each action is associated with a set of possible outcomes, and for each outcome we have the cost. At first, all outcomes are believed possible, but from real execution traces, it can learn some abstract actions are not possible.
Symbolic planning is performed to reach $G$, either by taking an optimistic semantic (the best outcome will happen), or using stochastic models taking into account the probabilities. In this work, we use $D^*$ as the incremental planner for its simplicity in analysis.

\subsection{$A^*$ and Dynamic $A^*$ ($D^*$)}
$A^*$ performs heuristic-guided search over the symbolic graph defined by the generative actions, expanding only the most promising nodes without constructing the full graph.
{\em Dynamic $A^*$} ({\em $D^*$}) \cite{stentz1995focussed,stentz1994optimal} extends $A^*$ to the InK setting: it plans using the current knowledge, and upon action failure, updates the graph and replans. Originally developed for discrete navigation, $D^*$ is known to be close to optimal {\em in the worst case} (by a factor at most $O(\sqrt{|V|})$, where $|V|$ is the number of graph vertices or abstract states) \cite{koenig2003performance}. However, no theoretical guarantees exist {\em on average over all possible worlds}. Critically, $D^*$ does not consider the belief set $\mathcal{W}_0$, making it unable to exploit prior knowledge. We address this limitation in Section~\ref{sec:bwts} with an algorithm that is {\em optimal on average} over $\mathcal{W}_0$.

\section{Belief World Tree search (BWTS)}
\label{sec:bwts}

As the incremental planner $D^*$ does not consider the belief set $\mathcal{W}_0$, it cannot exploit structural prior knowledge, making it suboptimal on average over possible worlds. We propose BWTS, an algorithm that is \textit{optimal on average} over $\mathcal{W}_0$, particularly effective when structural prior knowledge is available.
We use maze navigation as a running example to illustrate BWTS throughout this section.
BWTS is developed to provide {\em optimal} solutions for {\em discrete (symbolic) graphs} $(V, A, \tau_u, C)$, where $V$ is the set of vertices (representing discrete abstract states), $A$ is the finite set of actions, $\tau_u: V \times A \to V$ is the unobstructed transition function, and $C(v,a)$ is the cost of taking action $a$ from vertex $v$, considering the set ${\mathcal W}_0$ of possible worlds. In each world $W \in \mathcal{W}_0$, action $a$ from $v$ either reaches $v' = \tau_u(v,a)$ if unobstructed, or loops back to $v$ if a wall blocks $v'$.
BWTS computes a policy $\pi$, mapping history $\rho$ to action $\pi(\rho)=a$, minimizing the cumulative cost averaged over worlds $W\in\mathcal{W}_0$ (Eq.~(\ref{eq1})).
We write $v^a = \tau_u(v,a)$; e.g., $(2,3)^{\text{right}} = (3,3)$.

BWTS will maintain a belief set $\mathcal{W}_\rho \subseteq \mathcal{W}_0$
of worlds in $\mathcal{W}_0$ compatible with history $\rho$ ending in $v$. The belief set 
after history $\rho' = \rho,a,v'$ is 
$\mathcal{W}_{\rho'} = \mathcal{W}_{\rho}^- =
\{W \in \mathcal{W}_\rho \mid \text{ there is a wall in position }v^a\}$ if $v'=v$;
otherwise 
$\mathcal{W}_{\rho'}=\mathcal{W}_{\rho}^+ =
\{W \in \mathcal{W}_\rho \mid$ there is no wall in position $v^a=v'\}$. 
That is, $\mathcal{W}_\rho=  \mathcal{W}_{\rho}^+ \sqcup \mathcal{W}_{\rho}^-$ is partitioned into two sets. 

\begin{figure*}[t] 
  \centering
  \begin{subfigure}{.24\textwidth}
    \centering
    \includegraphics[width=\linewidth]{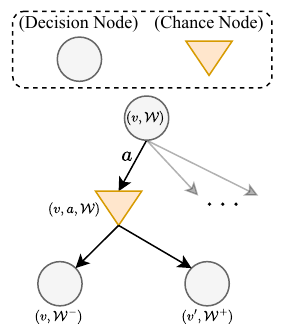}
    \caption{BWTS Nodes}
    \label{fig:BWTS_nodes}
  \end{subfigure} 
  \begin{subfigure}{.75\textwidth}
    \centering
    \includegraphics[width=\linewidth]{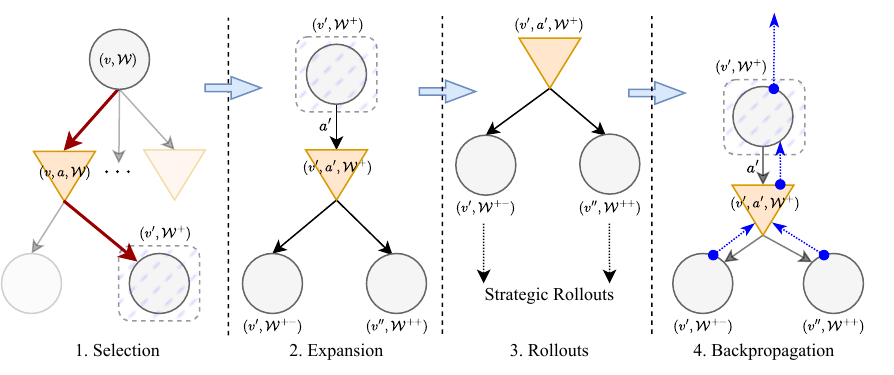}
    \caption{BWTS Learning stages}
    \label{fig:BWTS_learning}
  \end{subfigure}
  \caption{BWTS node types in (a) and learning stages in (b): 1. Selection (tree policy) $\to$ 2. Expansion (unvisited nodes) $\to$ 3.~Rollout (value estimate) $\to$ 4.~Backpropagation. Repeat these four stages over iterations to grow the tree.}
  \label{fig:BWTS_explanation}
\end{figure*}

\subsection{The complete BWTS tree}

We now describe BWTS trees.
The root is $(v_0,\mathcal{W}_0)$, with $v_0$ the initial state and 
$\mathcal{W}_0$ the set of all possible worlds.
Nodes are either decision nodes or chance nodes as illustrated in Fig. \ref{fig:BWTS_nodes}.
\emph{Decision nodes} represent choices of actions and are denoted by
$(v,\mathcal{W})$. A decision node $(v,\mathcal{W})$ 
has $|A|$ (chance nodes) children $(v,a,\mathcal{W})$, one per action $a \in A$. 
A chance node $(v,a,\mathcal{W})$ has two (decision nodes) children
$(v,\mathcal{W}^-)$ and $(v',\mathcal{W}^+)$
representing whether the unblocked successor $v^a$ from $v$ playing $a$ is a wall or not.
Decision nodes $(v,\mathcal{W})$ where 
$|\mathcal{W}|=1$, i.e. $\mathcal{W} = \{W\}$ for some $W$, 
have no successor. The world is fully known to be $W$, and a shortest-path algorithm (e.g., $A^*$ or Dijkstra) can be run on $W$. 
Notice that the path from the root to a node $(v_n,\mathcal{W}_n)$ of the tree describes the history $\rho=v_0, a_1, \ldots, v_n$, and we have $\mathcal{W}_n = \mathcal{W}_\rho$. 
If the BWTS is {\em complete}, that is unfolded until all the leaves $(v,\mathcal{W})$ are associated with belief $\mathcal{W}$ with a unique world, then we can associate each node with its cost, from the leaves to the root:
The cost $C(v,\{W\})$ of a leaf $(v,\{W\})$ is the cost of the shortest path in $W$ to the goal. 

\noindent
Then we inductively compute bottom up: for a chance node $(v,a,\mathcal{W})$ with two children  
$(v_1,\mathcal{W}_1)$, $(v_2,\mathcal{W}_2)$, its cost is:

\begin{equation}
\label{eq:chance}
  C(v,a,\mathcal{W})= \frac{|\mathcal{W}_1| C(v_1,\mathcal{W}_1) + |\mathcal{W}_2|C(v_2,\mathcal{W}_2)}{|\mathcal{W}_1|+|\mathcal{W}_2| = |\mathcal{W}|}    
\end{equation}

\noindent
For a decision node $(v,\mathcal{W})$, we define:
\begin{equation}
\label{eq:decision}
C(v,\mathcal{W}) = \min_{a \in A} C(v,a) + C(v,a,\mathcal{W})
\end{equation}

\begin{theorem}
\label{prop.complete}
The policy $\pi^*$ choosing in $(v,\mathcal{W})$ action $a$ minimizing $C(v,a,\mathcal{W})$ has average cost $C(\pi^*)=C(v_0,\mathcal{W}_0)$. Further, $C(\pi^*)=C^*$ the optimal cost of the InK problem.
\end{theorem}

\noindent
{\bf Evaluation of the size of the full BWTS tree:} 
Assuming $\mathcal{W}=  \mathcal{W}^+ \sqcup \mathcal{W}^-$ always partitions the belief set $\mathcal{W}$ into two sets of the same size, the number of nodes of the full BWTS tree is
at least $|\mathcal{W}_0|\times |A|^{\log{|\mathcal{W}_0|}} > 10^8$ 
for $|\mathcal{W}_0|=128$ and $|A|=8$. The proofs of Theorem~\ref{prop.complete} (by induction over the size of $|\mathcal{W}_0|$) and the evaluation can be found in the supplementary material.

\subsection{The BWTS algorithm}

Constructing the complete BWTS tree becomes intractable as the size of $|\mathcal{W}_0|$ grows (there is one node per history till a leaf). 
With good heuristics selecting promising actions, we can open only the most promising actions, and this will be sufficient to obtain an efficient policy $\pi$. Similar to Monte-Carlo Tree Search (MCTS), we approximate the cost of nodes without expanding them, and only expand the most promising ones, as illustrated in Fig.~\ref{fig:BWTS_explanation}. Unlike MCTS, which relies on random rollouts to evaluate the value of a node:

\smallskip
\noindent
\textbf{Strategic rollouts:}
Define $k$ fixed strategies $\sigma_1, \ldots, \sigma_k$ to reach the goal from any state $v$ and any belief set $\mathcal{W}$. 
$D^*$ is one such strategy. Variants can be used, like fixing the direction to favor when an obstacle is encountered, or targeting a particular sequence of intermediate subgoals. 
To evaluate any decision node $(v,\mathcal{W})$, we evaluate 
the cost $C(\sigma_i,v,W)$ of every strategy $\sigma_i$, $i \leq k$ from $v$ in every world $W \in \mathcal{W}$.
The value  $Q(v,\mathcal{W})$ we associate with $(v,\mathcal{W})$ is then: 

\begin{equation}
\label{eq.eval}
Q(v,\mathcal{W})=\min_{i \leq k} \frac{ \sum_{W \in \mathcal{W}} C(\sigma_i,v,W) }{|\mathcal{W}|} \geq 0
\end{equation}

We evaluate $Q(v,a,\mathcal{W})$ for chance nodes by computation from 
the valuation of its two (decision nodes) children.
Strategic rollouts have two advantages over standard Monte-Carlo rollouts: first, they limit the occurrence of loops around states. 
Also, rather than the average value of random rollouts, we consider a min over strategies of the average over worlds, closely matching the cost to evaluate, though we evaluate only a few strategies instead of the full BWTS.

\smallskip
\noindent
\textbf{Exploration bonus for chance nodes:}
Notice that for $k$ fixed strategies, $Q(v,a,\mathcal{W})$ is a fixed evaluation of the cost of a node $(v,a,\mathcal{W})$. Hence it is important to provide a way to not starve some nodes $(v,a,\mathcal{W})$ from being explored, because although their valuation seems less promising than a node $(v,a',\mathcal{W})$, they may ultimately lead to the optimal strategy. This is done by giving an advantage to chance nodes $(v,a,\mathcal{W})$ whose visit count $N(v,a,\mathcal{W})$ --- the number of times the node has been selected during search --- is much smaller than the visit count $N(v,\mathcal{W})$ of its parent decision node.
We use the same formula as in MCTS based on Lower Confidence Bound (LCB) rule~\cite{browne2012survey}.
First, we fix an exploration hyperparameter $c_e >0$ to control the exploration–exploitation trade-off: larger $c_e$ encourages more exploration (potentially better solutions) at the cost of additional compute. The action $\alpha(v,\mathcal{W})$ that will be explored by the BWTS algorithm at a decision node $(v,\mathcal{W})$ will be
\begin{equation}
\label{eq:treepol_or}
\alpha(v,\mathcal{W}) \;\in\; \arg\min_{a\in A}
\Big(Q(v,a,\mathcal{W})\;-\;c_{e}\sqrt{\tfrac{\ln N(v,\mathcal{W})}{N(v,a,\mathcal{W})}} \Big).
\end{equation}

\smallskip
\noindent
\textbf{Selection of decision nodes:}
At a chance node $(v,a,\mathcal{W})$, 
both decision-node children $(v,\mathcal{W}^-)$ and $(v',\mathcal{W}^+)$
should be opened to have a value to node $(v,a,\mathcal{W})$.
However, we know that the subtree with more worlds in the belief set
$\mathcal{W}^-$ or $\mathcal{W}^+$ will need more explorations
to find an efficient strategy than the one with fewer worlds, because evaluating {\em one} complete strategy in BWTS requires opening of a number of branches linear in the number of worlds in the belief set.
We have $\mathcal{W} = \mathcal{W}^- \sqcup \mathcal{W}^+$.
We write $(T,\mathcal{X})$ for either of the two children of $(v,a,\mathcal{W})$.
From $(v,a,\mathcal{W})$, 
the BWTS algorithm will explore $(T,\mathcal{X})$ following the selection rule:
\begin{equation}\label{eq:treepol_and}
\begin{aligned}
\omega(v,a,\mathcal{W})  &= \arg\min_{(T,\mathcal{X})}
\Big( N(T,\mathcal{X}) |\mathcal{W} \setminus \mathcal{X}| \Big).
\end{aligned}
\end{equation}

\smallskip
\noindent
\textbf{Stages of BWTS Construction:}
Every iteration of the BWTS starts at the root, which is a decision node $(v_0,\mathcal{W}_0)$.
We adopt in Algorithm~\ref{alg:BWTS} the standard MCTS~\cite{browne2012survey} stages (selection, expansion, rollout, and backpropagation) with modifications to 3.~Rollouts (see above), and to 4.~Backpropagation (see below).

\begin{algorithm}[t]
\caption{Belief World Tree search}
\label{alg:BWTS}
\begin{algorithmic}[1]
\Require root decision node $(v_0, \mathcal{W}_0)$, total number of iterations $I$, exploration constant $c_{e}$, rollout policies $(\sigma_{i})_{1 \leq k}$.
\State $root \gets (v_0, \mathcal{W}_0)$
\For{$i \in \{0, \dots , I\}$} \Comment{iterate $I$ times}
  \Statex \textbf{\#1. Selection}
  \State $u \gets root$
  \While{$u$ is fully expanded}
    \If{$u$ is a decision node}
      \State $a \gets \alpha(u)$; $u \gets \text{Child}(u,a)$ \Comment{Eq. \ref{eq:treepol_or}}
    \Else
      \State $u \gets \omega(u)$ \Comment{Eq. \ref{eq:treepol_and}}
    \EndIf
  \EndWhile
  \State $(v,\mathcal{W}) \gets u$

  \Statex \textbf{\#2. Expansion}
  \State randomly sample a child $(v,a,\mathcal{W})$ that does not exist in the tree and attach it to the tree
  \State attach $(v,\mathcal{W}^-)$ and $(v',\mathcal{W}^+)$ to the tree

  \Statex \textbf{\#3. Rollouts}
  \State Compute $Q(v',\mathcal{W}^+)$ and $Q(v,\mathcal{W}^-)$ \Comment{Eq. \ref{eq.eval}}

  \Statex \textbf{\#4. Backpropagation}
  \State $v \gets (v,\mathcal{W})$
  \While{$v \neq root$}
    \State update $Q(v)$ \Comment{Eq.~\ref{eq:chance_value} and \ref{eq:decision_value}}
    \State $v \gets parent(v)$
  \EndWhile
\EndFor
\Return $root$
\end{algorithmic}
\end{algorithm}

\smallskip
\noindent
\textbf{\emph{4.~Backpropagation.}}
After computing the evaluations $Q(v',\mathcal{W}^+)$ and $Q(v,\mathcal{W}^-)$, 
we backtrack inductively from the bottom to the root and update the values of ancestor nodes.
For chance nodes $(v, a, \mathcal{W})$ with children 
\((v, \mathcal{W}^-), (v', \mathcal{W}^+)\), 
we set as in Eq.~(\ref{eq:chance}):
\begin{equation}
 \label{eq:chance_value}
Q(v,a,\mathcal{W})= \frac{|\mathcal{W}_1| Q(v_1,\mathcal{W}_1) + |\mathcal{W}_2|Q(v_2,\mathcal{W}_2)}{|\mathcal{W}_1|+|\mathcal{W}_2| = |\mathcal{W}|}    
\end{equation}

\noindent
For decision nodes \((v, \mathcal{W})\),
we set as in Eq.~(\ref{eq:decision}):
\begin{equation}
\label{eq:decision_value}
Q(v,\mathcal{W}) \;\leftarrow\; \min_{a \in A} (C(v, a) \;+\; Q(v,a,\mathcal{W}))
\end{equation}

\medskip
\noindent
\textbf{\emph{BWTS Policy $\pi_I$}:}
After $i$ iterations, the strategy $\pi_i$ is the following: at each decision node, pick the action with the smallest $Q(v,a,\mathcal{W})$.  
We show in Theorem~\ref{thm:BWTS-converges} that $\pi_{i}$ converges to an optimal policy for the (symbolic) InK problem.

\begin{theorem}[Convergence of BWTS to Full-Tree]
\label{thm:BWTS-converges}
Let $(\pi_i)_{i \in \mathbb{N}}$ 
be the sequence of policies produced by the BWTS algorithm.
Then the sequence $C(\pi_i)_{i \in \mathbb{N}}$ converges towards $C^*$, the optimal cost of the (discrete symbolic) InK problem.
\end{theorem}

The proof can be found in the supplementary materials. Sketch: Let $BWTS_i$ be the tree after $i$ iterations of Algorithm~\ref{alg:BWTS}
for all  $i \in \mathbb{N}$.
We prove that there exists a number $I$ of iterations after which  $BWTS_i,i>I$ is the complete BWTS tree. 
Henceforth, $C(\pi_i)=C^*, i>I$, following Theorem~\ref{prop.complete}. 
Compared to the full BWTS tree, the 
smallest number of iterations (resp. nodes) is 
$I_0 = |A|\times |\mathcal{W}_0|$ 
(resp. min\_nodes $= 3 I_0$), to reach the leaves 
(and thus evaluate the cost) on all the (uncontrollable) chance branches. Branches from decision nodes are only sparsely explored.
The algorithm needs $O(I \times k \times |V| \times |\mathcal{W}_0|)$ operations to compute $\pi_I$.



\section{Experimental Evaluation}\label{results}

In this section, we outline the experimental setup and evaluate the performance of InK in comparison to the non-InK approach. All experiments were implemented in Python/PyTorch and run on Ubuntu 20.04.6 LTS on a machine with 8P+12E CPU cores (Intel i7-14700KF) and 64 GB of RAM. 
Source code is available at \url{https://github.com/CPS-research-group/ink_bwts}.
To evaluate InK on HRL tasks, we consider a setting where the low-level policy is pre-trained independently before being used to learn the abstract world model. We use the maze navigation task from \cite{rgl2024}, where a goal-conditioned low-level policy pre-trained in a clean maze learns the full abstract world model. Unlike InK, this baseline requires extensive exploration. Note that we selected $D^*$ because the domain is a navigation task; it can be replaced with any other incremental planner depending on the task.
Concretely, we investigate the following:

\smallskip
\noindent
\textbf{InK vs. Non-InK.}
We examine the performance, sample efficiency, and execution time of our proposed InK neurosymbolic HRL in a continuous physical environment, using $D^*$ as the high-level symbolic planner, and compare it with a non-InK HRL approach, RGL~\cite{rgl2024}, which first learns a full model of the concrete system before planning (Section~\ref{res:nesy_hrl}).
We compare on the RGL environments and report only the RGL results, since it is the state-of-the-art non-InK approach.

\smallskip
\noindent
\textbf{Comparison of symbolic planners.}
We evaluate $D^*$, BWTS, and BAMCP~\cite{bamdp2013} on small symbolic environments.
We analyze the impact of prior knowledge, and test also in an InK complex continuous environment in Section~\ref{res:bwts_performance}.

\subsection{InK vs. Non-InK \label{res:nesy_hrl}}
We experiment on three noisy physical Point-Maze environments from RGL \cite{rgl2024}: Four Rooms, Medium Maze, and Hard Maze, where a point agent navigates a planar maze. The action is a continuous 2D position increment $a \in [-1,1]^2$, transitions are stochastic with additive Gaussian noise, and the state is the agent's Cartesian position $s \in \mathbb{R}^2$.
We pretrain the low-level goal-reaching policy in an obstacle-free environment, following the experimental setup of RGL~\cite{rgl2024}.
Both InK and RGL share the same low-level policy. 
The key difference lies in high-level planning, where non-InK methods construct a complete world graph upfront (training) in RGL, whereas InK uses $D^*$ planning at the high-level, and can start reaching a goal from scratch ($\emptyset$). No prior knowledge is used.

We use five different seeds to train the low-level policy (for both InK and non-InK), and for each seed we sample ten random start–goal pairs. We average over these 50 experiments. We first compare in Table \ref{tab:bwts_mazephys_transposed} the number of steps (and time) required to reach the goal for the first time, either from scratch ($\emptyset$), meaning in the training part for RGL. We also compare after training for RGL (which takes extensive time $\geq 500 s, 100000$ samples), and after $75$ random goals have been sampled and reached for InK, which prebuilds a knowledge base, and is much more efficient ($<1s,5000$ samples).
Furthermore, we demonstrate the applicability of InK to an even more complex, high-dimensional domain (29 dimensions): the AntMaze U-Room environment~\cite{todorov2012mujoco}, in which an ant agent with multiple actuated joints (from MuJoCo) navigates a U-shaped corridor. The low-level goal-reaching policy is pre-trained in the same manner as described above.

\smallskip
\noindent
\textbf{Discussion:}
InK requires 30 to 100 times fewer samples than RGL to reach the goal from scratch ($\emptyset$), and after ``training" (reaching 75 goals), InK remains more efficient than RGL in both number of samples (18\% less) and time (20x). Moreover, the ``training'' phase of InK itself is more sample-efficient than RGL training. Over a sequence of goals, which RGL is specifically designed to optimize for, RGL initially benefits from its learned structure, as shown by the initial dip in Fig.~\ref{fig:bwts_vs_rgl}. However, after approximately 25 goals, InK becomes as efficient as RGL and then increasingly more efficient, with the cumulative sample gap exceeding 100,000 and growing further in the case of Hard Maze. InK is consistently better than RGL, even on RGL environments without prior knowledge.

\begin{table}[t]
  \centering
  \caption{InK vs RGL number of samples (\& runtime), $\downarrow$ lower is better, to reach goal from start for 50 different start-goal pairs.}
  \label{tab:bwts_mazephys_transposed}
  \begin{tabular}{@{}lccc@{}}
    \toprule
    Method & Four Rooms & Medium Maze & Hard Maze \\
    \midrule

    \multirow{2}{*}{RGL from scratch ($\emptyset$)}
          & $2080 \pm 701$   & $10412 \pm 3662$            & $17456 \pm 5585$ \\
          & $(96.48s \pm 32.38)$    & $(1436s \pm 608)$     & 
          $(1018s \pm 362)$ \\


    \multirow{2}{*}{InK from scratch ($\emptyset$)}
      & $\mathbf{64.90 \pm 9.17}$    & $\mathbf{147.82 \pm 28.29}$    & $\mathbf{208.92 \pm 41.20}$ \\
      & $\mathbf{(0.02s \pm 0.00)}$     & $\mathbf{(0.05s \pm 0.01)}$       & $\mathbf{(0.04s \pm 0.01)}$ \\

    \midrule

    \multirow{2}{*}{RGL after train}
      & $40.08 \pm 2.97$ & $48.32 \pm 4.38$ & $69.82 \pm 20.74 $ \\
      & $(0.34s \pm 0.02)$  & $(0.28s \pm 0.01)$  & $(0.21s \pm 0.01)$ \\

    
    InK after 75 goals
      & $\mathbf{31.28 \pm 8.59}$    & $\mathbf{40.36 \pm 21.06}$    & $\mathbf{58.16 \pm 18.74}$ \\
      \quad (after ``train'') & $\mathbf{(0.01s \pm 0.002)}$     & $\mathbf{(0.01s \pm 0.006)}$       & $\mathbf{(0.01s \pm 0.003)}$ \\

    \midrule

      \multirow{2}{*}{RGL full training}
      & $100000$ & $150000$ & $300000$ \\
      & $(498.63s \pm 59.56)$  & $(1167.48s \pm 77.82)$  & $(6022.96s \pm 2063.39)$ \\


      InK completing 75 goals
      & $\mathbf{2936.0 \pm 334.36}$ & $\mathbf{3095.4 \pm 251.20}$ & $\mathbf{4651.2 \pm 1396.33}$ \\
      \quad (``training'') & $(\mathbf{0.84s \pm 0.22)}$  & $(\mathbf{0.87s \pm 0.176})$  & $(\mathbf{0.83s \pm 0.25})$ \\

    \bottomrule
  \end{tabular}
\end{table}

\begin{figure}[]
  \centering
  \begin{subfigure}[]{0.32\textwidth}
    \centering
    \includegraphics[width=\linewidth]{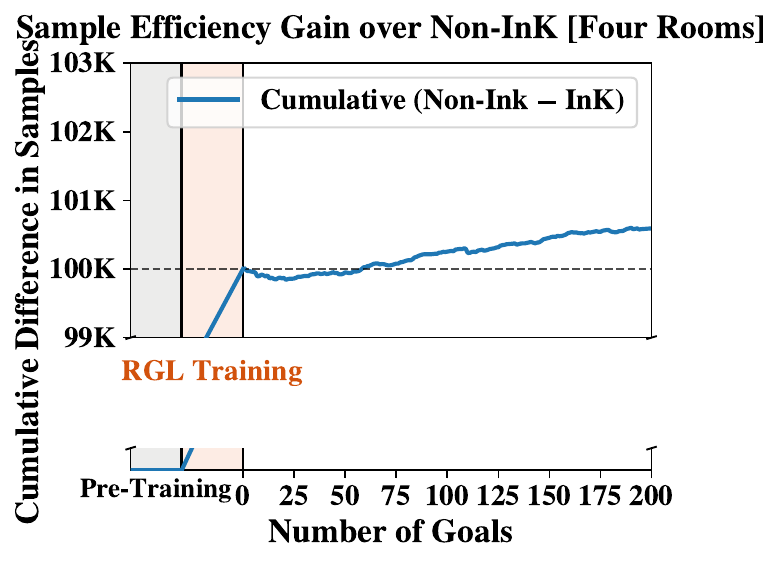}
    \caption{Four Rooms}
    \label{fig:pm_fourrooms}
  \end{subfigure}
  \begin{subfigure}[]{0.32\textwidth}
    \centering
    \includegraphics[width=\linewidth]{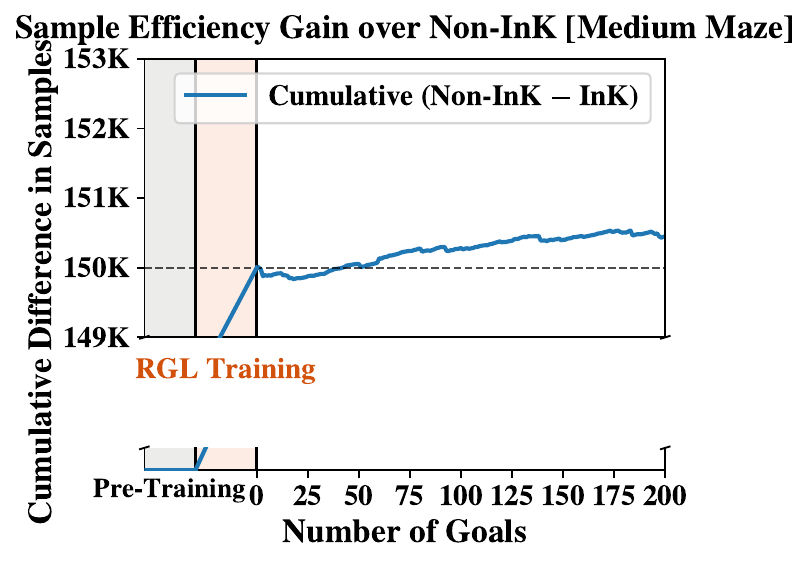}
    \caption{Medium Maze}
    \label{fig:pm_mediummaze}
  \end{subfigure}
  \begin{subfigure}[]{0.32\textwidth}
    \centering
    \includegraphics[width=\linewidth]{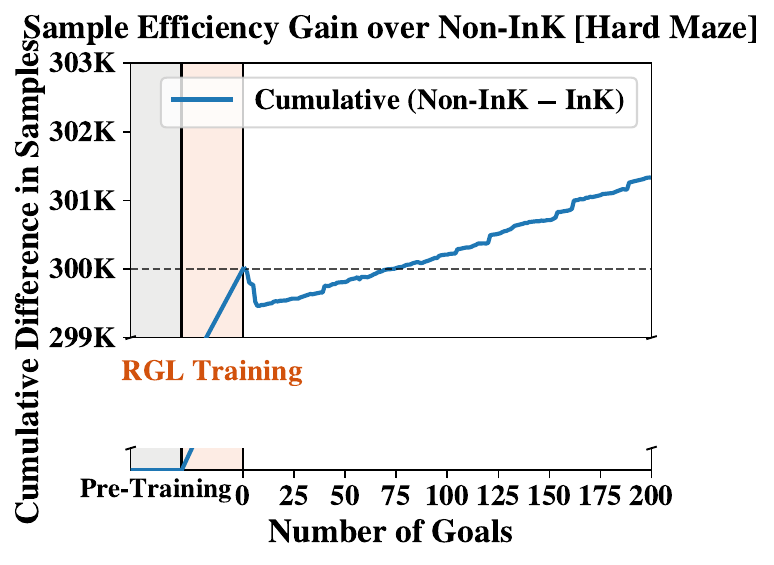}
    \caption{Hard Maze}
    \label{fig:pm_hardmaze}
  \end{subfigure}
  \caption{Sample-efficiency gain of InK over RGL ($>$100K environment steps) in the Point Maze environment. Pre-Training: learning the low-level neural module (same); RGL Training: learning the full knowledge of the concrete world.}
  \label{fig:bwts_vs_rgl}
\end{figure}

\noindent
\textbf{InK in Ant-Maze U-Room}:
InK can handle efficiently even more challenging environments, such as a MuJoCo Ant agent in a U-shaped maze, namely the Ant-Maze U-Room~\cite{todorov2012mujoco}. Compared with the previous benchmark, the (neural) low-level policy has to handle a much more complex physical environment (29 dimensions). We pretrain this neural low-level policy and use the same low-level policy for both InK and non-InK. As above, Non-InK requires an overwhelming number of samples ($\sim$10K) to build the full graph of the abstract world, whereas InK reaches the goal while building a goal-directed InK abstract world using around 1507 steps (with a $D^*$ high-level planner and even less, 1134 steps, with a BWTS high-level planner).

\begin{figure}[h!] 
  \centering
  \begin{subfigure}{0.32\linewidth}
    \centering
    \includegraphics[width=0.6\linewidth]{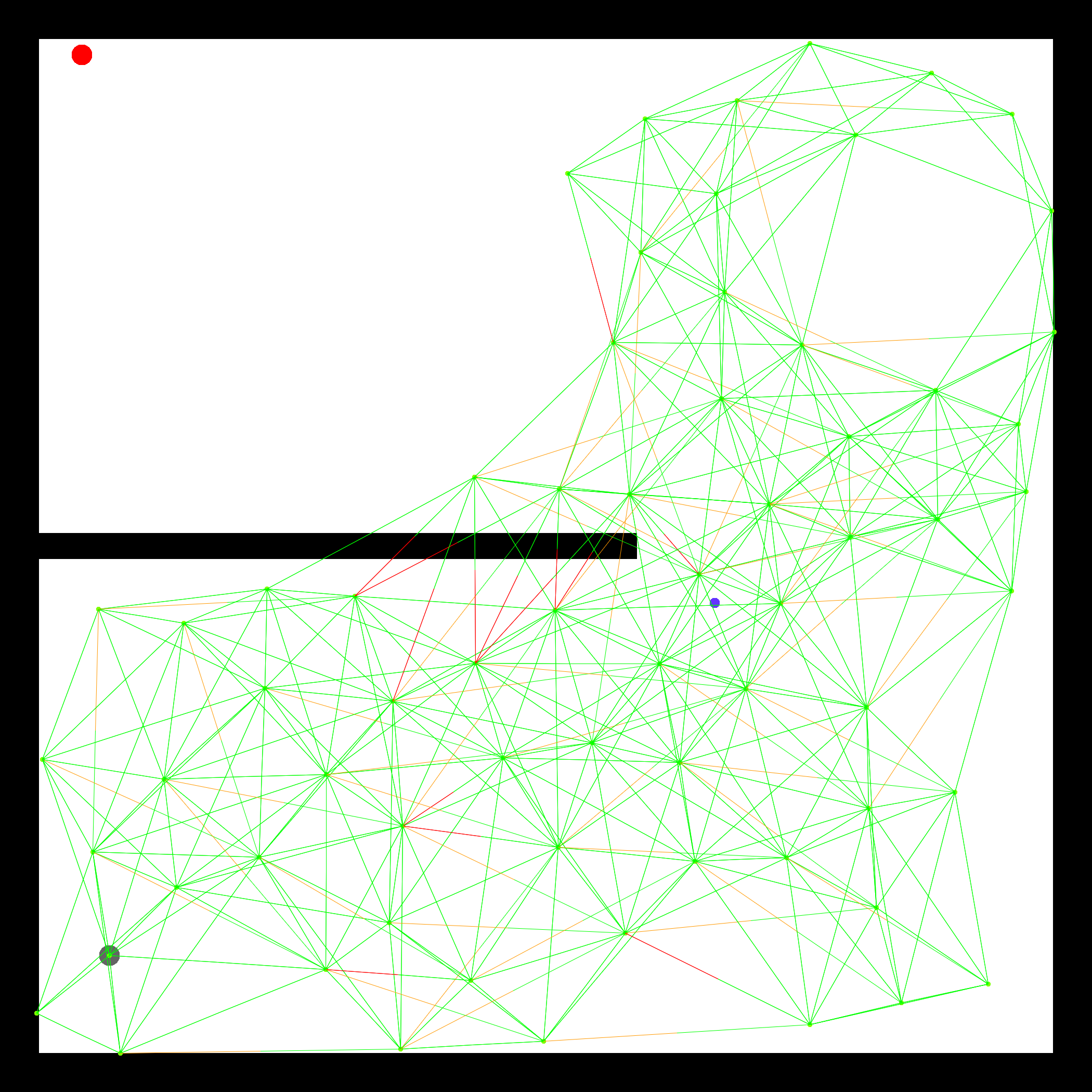}
    \caption{Non-InK (RGL)}
    \label{fig:ant_maze_rgl}
  \end{subfigure}
  \begin{subfigure}{0.32\linewidth}
    \centering
    \includegraphics[width=0.6\linewidth,trim={0 0 0 4mm},clip]{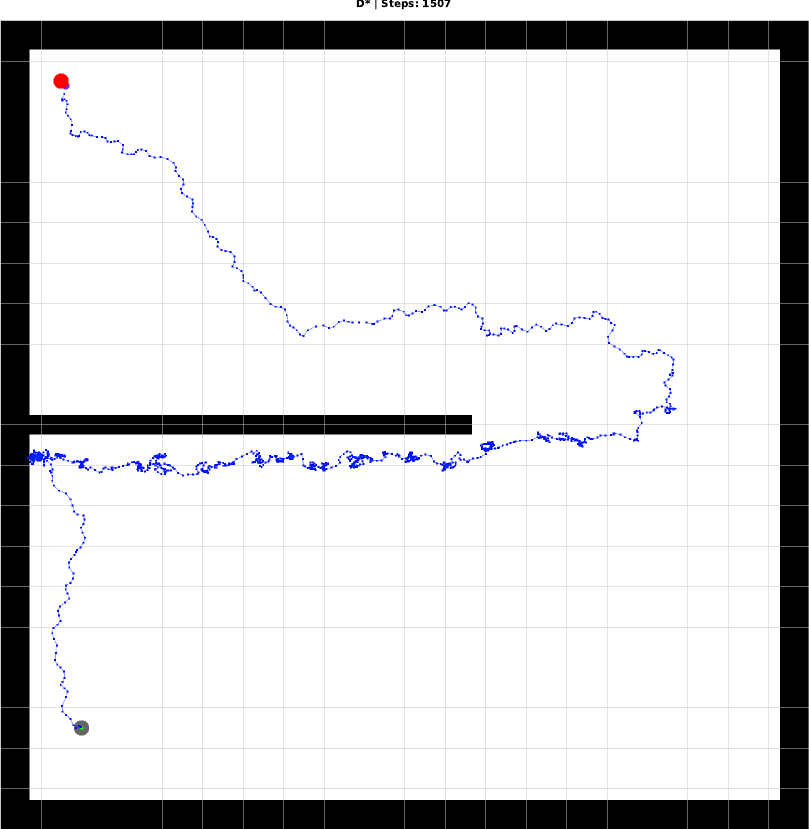}
    \caption{InK-$D^*$}
    \label{fig:ant_maze_dstar}
  \end{subfigure}
  \begin{subfigure}{0.32\linewidth}
    \centering
    \includegraphics[width=0.6\linewidth,trim={0 0 0 4mm},clip]{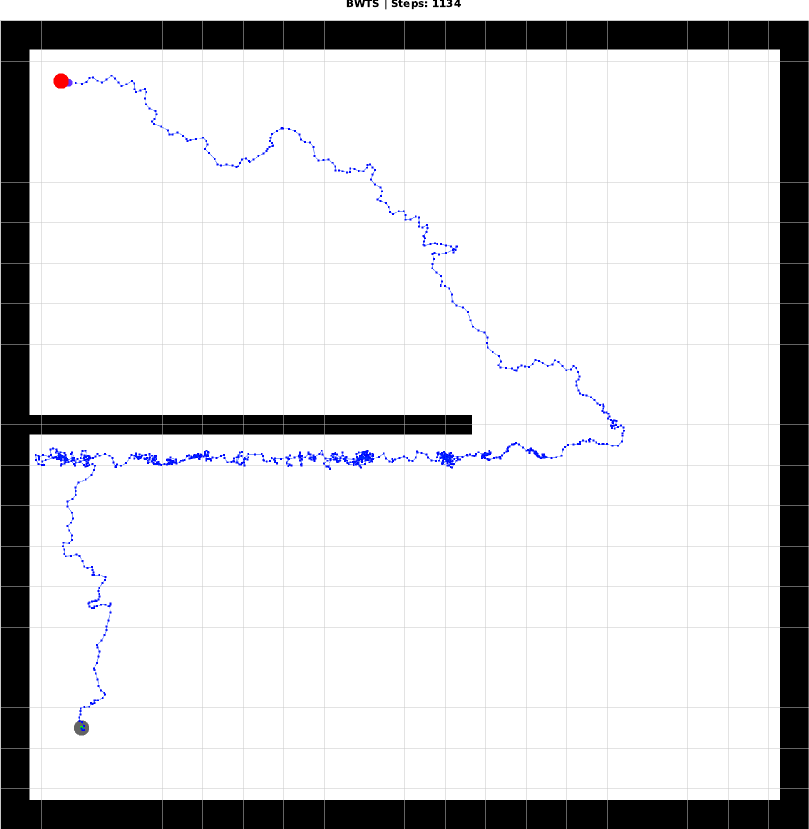}
    \caption{InK-BWTS}
    \label{fig:ant_maze_bwts}
  \end{subfigure}
  \caption{Ant-Maze U-Room RL environment. \textcolor{blue}{Blue} dot: start; \textcolor{red}{red} dot: goal.}
  \label{fig:ant_maze}
\end{figure}

\subsection{Comparison between symbolic planners} \label{res:bwts_performance}

In this section, we compare symbolic planners $D^*$, BWTS, and BAMCP~\cite{bamdp2013}, with and without prior structural knowledge. We first consider simple synthetic environments with the full set $\mathcal{W}_0$ of possible worlds: $10\times10$ grids with a single horizontal or vertical wall, each having one opening. Let $\mathcal{W}_h$ be $10\times10$ worlds with a horizontal wall at each admissible row ($y=1$–$8$, excluding $0,9$ for start and goal) and an opening at each wall position ($x=0$–$9$), giving $|\mathcal{W}_h|=80$. The vertical-wall set $\mathcal{W}_v$ is defined analogously, and $\mathcal{W}_{hv} = \mathcal{W}_h \cup \mathcal{W}_v$. Details on belief set construction are in the supplementary materials.

For BWTS, we use an exploration bonus of $c_e=1.5$ and $I=200 000$ iterations for all belief sets.
The rollout strategies $\sigma_k$ combine five primary targets (the goal or one of four border midpoints) with two wall-sweeping variants, yielding $10$ strategies in total.
For comparison, we run $D^*$ independently on each belief world.
For BAMCP, the structural knowledge of $\mathcal{W}_{h,v,hv}$ is particularly hard to encode, with strong dependencies between different positions in the grid: if there is a wall in cell $(3,3)$, then there cannot be a wall in cell $(7,7)$. We first experimented with a prior distribution with probability $0.9$ to not be a wall and $0.1$ to be a wall, matching the probabilities in $\mathcal{W}_{h,v,hv}$. However, stochastic rollouts in such an open environment seldom find the goal (lots of loops), and BAMCP is very unstable 
(examples in supplementary material). Remember BWTS uses {\em strategic} rollouts instead of stochastic rollouts.
The probabilities we use are $.6/.4$ for a cell to be no wall/ a wall. 
The average results over all the belief worlds are summarized in Table~\ref{tab:bwts_vs_dstar}. 
As $D^*$ is deterministic, results are always the same on these deterministic symbolic environments.
For BWTS and BAMCP, we report the mean $\pm$ standard deviation over 10 different random seeds.

\begin{table}[b!]
  \centering
  \caption{Expected number of steps (\& runtime), $\downarrow$ lower is better, for the BWTS-learned policy, BAMCP, and $D^*$.}
  \label{tab:bwts_vs_dstar}
  \begin{tabular}{@{}lccc@{}}
    \toprule
    Belief Worlds & $\mathcal{W}_h$ & $\mathcal{W}_v$ & $\mathcal{W}_{hv}$ \\
    \midrule
    
    \multirow{2}{*}{$D^*$}
      & $24.56$ & $23.19$ & $23.88$ \\
      & $\mathbf{(0.01s)}$ & $\mathbf{(0.01s)}$ & $\mathbf{(0.01s)}$ \\

    \cmidrule(lr){2-4}

    \multirow{2}{*}{BWTS}
      & $\mathbf{21.22 \pm 0.20}$ & $\mathbf{21.22 \pm 0.16}$ & $\mathbf{23.36 \pm 0.08}$ \\
      & $(1.05s \pm .13)$             & $(1.06s \pm .1)$             & $(52s \pm 1.6)$ \\

    \cmidrule(lr){2-4}

    \multirow{2}{*}{BAMCP}
      & $28.42 \pm 1.16$ & $28.89 \pm 1.29$ & $40.42 \pm 1.39$ \\
      & $(30.88s \pm 1.65)$    & $(28.9s \pm 1.05)$            & $(92.1s \pm 2.19)$ \\

    \bottomrule
  \end{tabular}
\end{table}

\smallskip
\noindent
\textbf{Discussion}: 
Table~\ref{tab:bwts_vs_dstar} shows that BWTS achieves lowest expected cost. The gains are largest for belief worlds with only horizontal ($15.7\%$ vs $D^*$) or vertical walls ($9.3\%$ vs $D^*$).
For $\mathcal{W}_{hv}$, the margin becomes negligible, at $2.2\%$.
To understand the source of performance gain for BWTS in $\mathcal{W}_h$, we analyzed the trajectories (from start to goal) for two belief worlds (one in blue and one in black) in $\mathcal{W}_h$ as shown in Fig.~\ref{fig:bwts_vs_ds_wh}.
These examples demonstrate that BWTS exploits belief over possible doorway locations: the agent proceeds straight and initiates a directed wall sweep from the start, probing for the opening without costly backtracking. In a single instance like the blue world in Fig.~\ref{fig:bwts_vs_ds_wh}, $D^*$ may appear optimal because it commits to the correct sweep direction. However, in the black world, $D^*$ chooses the wrong initial sweep and then must backtrack, incurring extra cost. By planning over belief worlds, BWTS avoids this return sweep by adopting a probing strategy that minimizes expected cost \emph{before} the opening is revealed. In $\mathcal{W}_{hv}$, there is no efficient ``no backtracking'' strategy, and the runtime of BWTS becomes excessive for the negligible gain in number of steps.

\begin{figure}[t!]
  \centering
    \centering
    \includegraphics[width=0.26\linewidth]{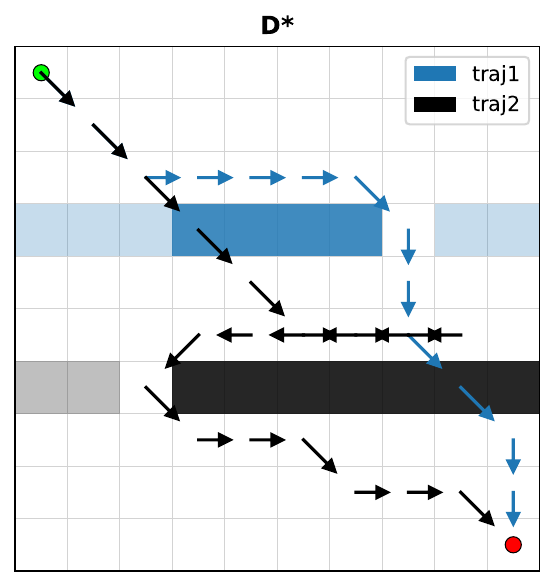}%
    \includegraphics[width=0.26\linewidth]{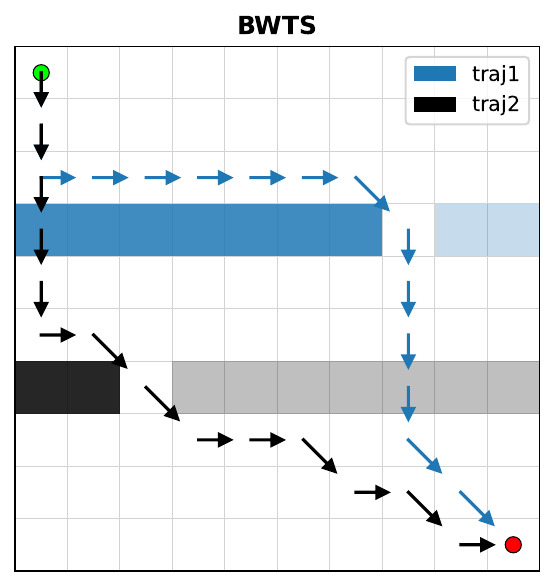}
    \includegraphics[width=0.26\linewidth]{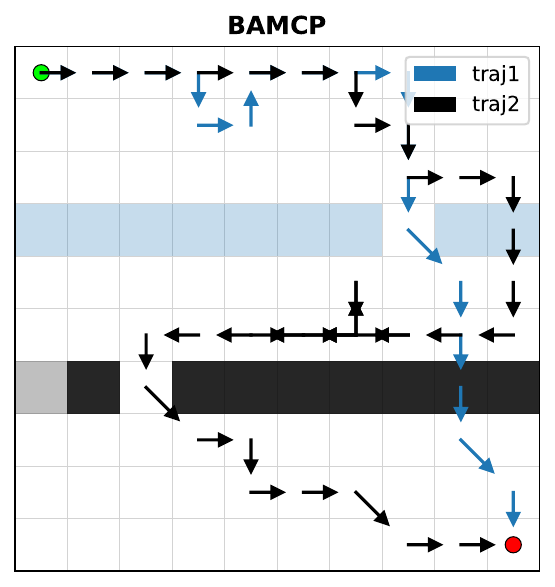}

  \caption{$D^\ast$, BWTS, and BAMCP for $W \in \mathcal{W}_h$ on {\color{blue}Blue} and Black Worlds. Greyed-out walls: unseen; highlighted walls: discovered en route.}
  \label{fig:bwts_vs_ds_wh}
  \vspace{-0.2cm}
\end{figure}

BAMCP reveals a stochastic and inconsistent strategy, sometimes making unnecessary detours, with an expected cost exceeding 28 steps compared to $D^*$'s average of under 25 steps, and a prohibitive runtime.

\begin{table}[t]
  \centering
  \caption{InK-BWTS vs InK-$D^*$ number of samples (\& runtime) to reach goal from start for 10 different start-goal pairs.}
  \label{tab:bwts_vs_dstar2}
  \begin{tabular}{@{}lc@{}}
    \toprule
    Method & Physical RGL Environments \\
    \midrule

    \multirow{2}{*}{InK-$D^*$}
      & $87.70 \pm 15.36$  \\
      & $(\mathbf{0.02s \pm 0.00})$ \\

    \multirow{2}{*}{InK-BWTS no prior Knowledge}
      & $\mathbf{86.97 \pm 13.21}$ \\
      & $(62.56s \pm 30.3)$ \\

    \midrule

    \multirow{2}{*}{\begin{tabular}{@{}l@{}}
        InK-BWTS w/ prior Knowledge \\
        construction time: ($2.66s \pm 0.88$) +
        \end{tabular}}
      & $\mathbf{45.97 \pm 22.97}$ \\
      & $(\mathbf{0.01s \pm 0.01})$ \\

    \bottomrule
  \end{tabular}
\end{table}

We finally compare in Table \ref{tab:bwts_vs_dstar2} both $D^*$ and BWTS as symbolic planners in the high-level of InK, on {\em resized} physical environments of RGL (so that the 3 point mazes have the same size). We compared without prior knowledge, with BWTS using the symbolic set $\mathcal{W}_{hv}$ of worlds to encode the position of the next wall (when a wall is met, we rerun BWTS from the full set $\mathcal{W}_{hv}$ to model the position of the following wall). 
InK-BWTS is slightly more efficient than InK-$D^*$, although runtime is much longer,
consistent with Table \ref{tab:bwts_vs_dstar} ($\mathcal{W}_{hv}$).
With prior knowledge, BWTS is initialized with the set $\mathcal{W}_{3}$ consisting of three worlds: 4 Rooms, Medium Maze, and Hard Maze. 
With $\mathcal{W}_{3}$, BWTS is constructed only once, and it is then used across the mazes with much faster execution time ($2.66s$). 
With prior knowledge, InK-BWTS is much more efficient than InK-$D^*$, halving the number of samples. Overall, when there is {\em no} prior structural knowledge, InK-$D^*$  is preferable.
With prior structural knowledge, InK-BWTS is more efficient due to BWTS’s optimality, at the cost of increased run time.

\section{Conclusion}\label{conclusion}
This paper proposes a {\em neurosymbolic} HRL framework with incremental knowledge (InK), where a symbolic high-level planner guides a neural low-level controller. The high-level planner uses InK for planning (e.g. with $D^*$), and directs the low-level controller to associated subgoals. 
 It combines the advantages of purely symbolic InK framework (e.g. Geometric Motion Planners \cite{tamper24,noseworthy2021active}) in terms of sample complexity, and of non-InK neurosymbolic framework (e.g. SORB \cite{eysenbach2019sorb} and RGL \cite{rgl2024}) that can handle complex model-free physical environments. It generates optimal strategies in complex physical environments with sparse-reward (e.g. a MuJoCo ant agent navigating a Maze), while staying efficient in terms of number of samples.
Further, when prior knowledge on the world is available, we develop the BWTS algorithm, which searches for an optimal policy averaged over belief sets. Empirically, in terms of sample efficiency, InK(-$D^*$) outperforms RGL, a non-InK method, from scratch ($\emptyset$), asymptotically, and in terms of training, without using prior knowledge. When prior knowledge is available, InK-BWTS further outperforms InK-$D^*$, with increased compute time as a trade-off.
Overall, these results demonstrate that neurosymbolic HRL with InK provides a practical approach to sample-efficient learning, yielding gains exceeding 100K environment steps in Point Maze.
We demonstrated the advantage of InK using navigation as a case study. For future work, we plan to extend InK to other domains such as manipulation using PDDL-based incremental planners. Additionally, while BWTS handles structural prior knowledge and BAMCP handles probabilistic priors, neither addresses both; combining these complementary strengths into a unified algorithm (as BAWMCP) is a promising direction.

\begin{credits}
\subsubsection{\ackname} This research was conducted as part of the DesCartes program and was supported by the National Research Foundation, Prime Minister’s Office, Singapore, under the Campus for Research Excellence and Technological Enterprise (CREATE) program. This research/project is also supported by the National Research Foundation, Singapore and DSO National Laboratories under the AI Singapore Programme (AISG Award No: AISG2-RP-2020-017). 
The second author is partly supported by ANR-23-PEIA-0006 SAIF.

\subsubsection{\discintname}
The authors have no competing interests to declare that are
relevant to the content of this article.
\end{credits}

%
%
%
\bibliographystyle{splncs04}
\bibliography{references}
%

\newpage
\appendix
\setcounter{theorem}{0}
\section{Theoretical Analysis on BWTS}

\begin{theorem}
\label{thm:1}
The policy $\pi^*$ choosing in $(v,\mathcal{W})$ action $a$ minimizing $C(v,a,\mathcal{W})$ has average cost $C(\pi^*)=C(v_0,\mathcal{W}_0)$. Further, $C(\pi^*)=C^*$ the optimal cost of the InK problem.
\end{theorem}

\begin{proof}[Sketch of Proof of Theorem 1.]
The proof is by induction on $|\mathcal{W}|$.
For  $|\mathcal{W}|=1$, the proof is trivial.

For the inductive step:
$\pi^*$ plays in decision nodes $(v,\mathcal{W})$ the action $a$ to minimize
$C(v,a,\mathcal{W})$.
Consider the chance node $(v_0,a_0,\mathcal{W}_0)$ with $a_0$ picked by $\pi^*$, and its two decision node children $(v_0,\mathcal{W}_0^-)$ and $(v_0^{a_0},\mathcal{W}_0^+)$.
We can apply the induction hypothesis on both $\mathcal{W}_0^-$ and
$\mathcal{W}_0^+$ which are strictly smaller than $\mathcal{W}_0$, to obtain that 
$\pi^*$ is optimal on both subtrees. 
By definition, we have $C(\pi^*)=C(v_0,\mathcal{W}_0)$.

Take any strategy $\pi'$.
Then consider the action $a'$ picked by $\pi'$ at the root, 
and $a$ picked by $\pi^*$.
If $C(v_0,a',\mathcal{W}_0) = C(v_0,a,\mathcal{W}_0)$,
then we can change $\pi^*$ to picking $a'$ at the root instead of $a$ without changing $C(\pi^*)$. So we can assume wlog that $a=a'$. We can then apply the induction hypothesis to both subtrees of $(v_0,a',\mathcal{W}_0)$ to find that 
$C(\pi') \geq C(\pi^*)$.
The last case is that $C(v_0,a',\mathcal{W}_0) > C(v_0,a,\mathcal{W}_0)$, by choice of $\pi^*$ minimizing $C(v_0,a,\mathcal{W}_0)$. Again, applying the induction hypothesis on both subtrees of $C(v_0,a',\mathcal{W}_0)$, we obtain that the best policy for $C(v_0,a',\mathcal{W}_0)$
is strictly worse than for $C(v_0,a,\mathcal{W}_0)$, and $C(\pi')>C(\pi^*)$.
So $C^* = \inf_{\pi'} C(\pi') \geq C(\pi^*) \geq C^*$ and we are done.
\end{proof}

\noindent
{\bf Evaluation of the size of the full BWTS tree:} 
Assuming $\mathcal{W}=  \mathcal{W}^+ \sqcup \mathcal{W}^-$ always partitions the belief set $\mathcal{W}$ into two sets of the same size, the number of nodes of the full BWTS tree is
at least $|\mathcal{W}_0|\times |A|^{\log{|\mathcal{W}_0|}} > 10^8$ 
for $|\mathcal{W}_0|=128$ and $|A|=8$. 

Consider the skeleton of the full BWTS tree, with only chance nodes: the tree is a balanced complete binary tree of depth $\log{|\mathcal{W}_0|}$, because after $\log{|\mathcal{W}_0|}$ chance nodes, there is a unique world $W$ and the node is a leaf. 
With all decision possibilities, 
there are $|A|$ choice nodes children of the first chance node. 
Assuming after one choice there is a partition of $\mathcal{W}$, we have $|A|$ times the skeleton of depth $\log{|\mathcal{W}_0|-1}$. At this end, it gives at least 
$|\mathcal{W}_0|\times |A|^{\log{|\mathcal{W}_0|}} = |\mathcal{W}_0|^{1+\log(|A|)}$ nodes.

The less optimistic case is that a choice node does not provide any new information and the set $\mathcal{W}$ stays the same. 
We can regroup different nodes with the same $(v,\mathcal{W})$, as they will have the same future.
So there can be at most $|V|$ nodes $(v,\mathcal{W})$ with the same belief set $\mathcal{W}$.
This gives at most 
$|\mathcal{W}_0|\times |V|^{\log{|\mathcal{W}_0|}} = 
|\mathcal{W}_0|^{1+\log(|V|)}$
nodes.

\begin{theorem}[Convergence of BWTS to Full-Tree]
\label{thm:BWTS-converges}
Let $(\pi_i)_{i \in \mathbb{N}}$ 
be the sequence of policies produced by the BWTS algorithm.
Then the sequence $C(\pi_i)_{i \in \mathbb{N}}$ converges towards $C^*$, the optimal cost of the (discrete symbolic) InK problem.
\end{theorem}

\begin{proof}[Proof of Theorem 2.]
Let $BWTS_i$ be the tree after $i$ iterations of the BWTS algorithm, 
for all  $i \in \mathbb{N}$.
We will prove that there exists a number $I$ of iterations after which  $BWTS_i$ is
the complete BWTS tree for all $i>I$. Henceforth, $\pi_i$ will satisfy
$C(\pi_i)=C^*$ for all $i >I$, following Theorem 1.

We claim that for all $i$, all nodes of $BWTS_i$ will be visited in the selection step of Algorithm 1 infinitely many times.
We reason by contradiction: it means that there exists a node $u$
and an index $I_u$ after which $u$ is not visited.
Among all such nodes $u$, consider a node $u_0$ closest to the root, and $U_0 \neq \emptyset$ the finite set of siblings of $u_0$ (including $u_0$) which are visited only finitely many times. There exists $I_0$ so that after $I_0$, no node from $U_0$ is visited anymore. Let $v_0$ be the parent of $u_0$. By choice of $u_0$, 
$v_0$ is visited infinitely many times. 

If $v_0$ is a chance node, 
then $u_0$ is a decision node with only one sibling $u'_0$.
As $v_0$ is seen infinitely often and $u_0$ is not seen after $I_0$, it means that 
$u'_0$ is seen infinitely often. 
So $N(u_0)$ will be fixed after $I_0$, while 
$N(u'_0)$ will be increasing towards $+\infty$.
Following Eq. (6), the argmin will eventually select $\omega(v_0)=u_0$
after $I_0$, a contradiction with $u_0$ is not visited after $I_0$.

Similarly, if $v_0$ is a decision node, then $U_0$ is a set of chance nodes.
$N(v)$ tends towards $\infty$.
For all $u \in U_0$, $N(u)$ and $Q(u)$ are constant after $I_0$.
Hence $Q(u) - C_e \sqrt{\frac{\ln N(v)}{N(u)}}$ converges towards $-\infty$.
On the other hand, for every $u' \notin U_0$ sibling of $u_0$, 
$N(u')$ tends towards $\infty$.
At least one $u_0'$ has $N(u_0')> \frac{1}{|A|} N(v)$.
As $Q(u'_0)>0$, there exists $I_1$ with
$Q(u'_0) - C_e \sqrt{\frac{\ln N(v)}{N(u'_0)}} >-1$ for all $i > I_1$.
Hence there exists an $I_2 > \max(I_0,I_1)$,  
such that for all $i> I_2$, for all $u \in U_0$, 
$Q(u) - C_e \sqrt{\frac{\ln N(v)}{N(u)}} <-1$.
As $u'_0$ is visited infinitely often ($u'_0 \notin U_0$), there is an index $i >I_2$ so that $\alpha$ selects $u'_0$ at index $i$.
This is a contradiction with the argmin of Eq. (5), as any node $u \in U_0$
has lower $Q(u) - C_e \sqrt{\frac{\ln N(v)}{N(u)}} <-1 < Q(u'_0) - C_e \sqrt{\frac{\ln N(v)}{N(u'_0)}}$.

So in both cases, we get a contradiction, proving that for all $i$ all nodes of $BWTS_i$ will be visited infinitely often. 
We claim then that there exists an $i$ so that $BWTS_i$ is the complete BWTS.
Again, we prove that by contradiction. If this is not the case, then there is a node of the complete BWTS that is in no $BWTS_i$. Take such $u$ a node the closest to the root.
It means by definition that its parent $v$ is in some $BWTS_i$.
Hence by the above, $v$ is visited infinitely often, and with the same reasoning as above, $u$ will be eventually added to the tree, a contradiction. Hence there exists some $I$ so that $BWTS_i$ is full for all $i>I$. In this case, the definition of the backpropagation for $V$ being  the same as for the cost $C$, the strategy for $BWTS_i$ is the same as for the complete BWTS.
\end{proof}

\noindent
{\bf Evaluation of the complexity:}
Compared to the full BWTS tree, the 
smallest number of iterations (resp. nodes) is 
$I_0 = |A|\times |\mathcal{W}_0|$ (resp. min\_nodes $= 3 I_0$), 1024 for our example , to reach the leaves (and thus evaluate the cost) on all the (uncontrollable) chance branches. Branches from decision nodes are only sparsely explored.
The algorithm needs $O(I \times k \times |V| \times |\mathcal{W}_0|)$ operations to compute $\pi_I$.

Compared with the full BWTS, we do not need to explore all decision-nodes, but only a few.
The skeleton with only chance nodes is a 
 binary balanced tree of depth $\log{|\mathcal{W}_0|}$. With only one decision explored, it gives 
$2^{\log{|\mathcal{W}_0|}}=|\mathcal{W}_0|$ nodes, times $|A|$ possible nodes for their children. 
Concerning the complexity: rollouts take 
$O(k |V| |\mathcal{W}_0|)$ operations for each $i$: 
for each of $k$ simple strategies, each world in $\mathcal{W}_0$, each state in $V$ is accessed $O(1)$ times.
The backpropagation needs $O(\log(|\mathcal{W}_0|))$ operations, which is negligible.

\section{Description of all the stages of BWTS}

\noindent
\textbf{Stages of BWTS Construction:}
Every iteration of the BWTS algorithm starts at the root, which is a decision node $(v_0,\mathcal{W}_0)$.
We adopt 
in Algorithm~\ref{alg:BWTS} the standard MCTS stages (selection, expansion, rollout, and backpropagation) with modifications to 3.~Rollouts (see above), and to 4.~Backpropagation detailed below.


\begin{algorithm}[h]
\caption{Belief World Tree search}
\label{alg:BWTS}
\begin{algorithmic}[1]
\Require root decision node $(v_0, \mathcal{W}_0)$, total number of iterations $I$, exploration constant $c_{e}$, rollout policies $(\sigma_{i})_{1 \leq k}$.
\State $root \gets (v_0, \mathcal{W}_0)$
\For{$i \in \{0, \dots , I\}$} \Comment{iterate $I$ times}

  \Statex \textbf{\#1. Selection}
  \State $u \gets root$
  \While{$u$ is fully expanded}
    \If{$u$ is a decision node}
      
      \State \quad $a \gets \alpha(u)$; 
      $u \gets \text{Child}(u,a)$ \Comment{Eq. (5)}
    \Else  
       \State \, $ u \gets \omega(u)$ \Comment{Eq. (6)}
    \EndIf
  \EndWhile
  \State $(v,\mathcal{W}) \gets u$
  
  \Statex \textbf{\#2. Expansion}
  \State randomly sample a child $(v,a,\mathcal{W})$ that does not exist in the tree and attach it to the tree
  
\State attach $(v,\mathcal{W}^-)$ and $(v',\mathcal{W}^+)$ to the tree

  \Statex \textbf{\#3. Rollouts}  
  \State Compute $Q(v',\mathcal{W}^+)$ 
  
  \State \quad \, \quad and $Q(v,\mathcal{W}^-)$ 
  \Comment{Eq. (4)}
  
  \Statex \textbf{\#4. Backpropagation}
  \State $v \gets (v,\mathcal{W})$
  \While{$v \neq root$}
    \State update $Q(v)$ \Comment{Eq.~(7) and (8)} 
    
    \quad \, \quad \, \quad \State $v \gets parent(v)$
  \EndWhile  
\EndFor
\State \Return $root$
\end{algorithmic}
\end{algorithm}

The four stages are repeated for iterations $i \in \{0,1,\dots, I\}$, where the hyperparameter $I$ denotes the total number of iterations (compute budget) to construct BWTS. We say that a decision node $(v,\mathcal{W})$ is \emph{incomplete} if 
there is an action $a \in A$ without chance node $(v,a,\mathcal{W})$.
At each iteration, one incomplete decision node $(v,\mathcal{W})$ is {\em selected}, 
one chance node $(v,a,\mathcal{W})$ is added as well as its two (decision nodes) children
$(v',\mathcal{W}^+)$ and $(v,\mathcal{W}^-)$ ({\em expansion} stage).

\smallskip
\noindent
\textbf{\emph{(1) Selection.}}
Starting from the root $(v_0, \mathcal{W}_0)$, we descend the tree using the rules $\alpha,\omega$ defined in Eq. (5),(6),
until we reach an incomplete
decision node $(v,\mathcal{W})$ (or a depth/horizon cut-off). We increment the number of visits counter $N$ for each node visited during the descent.

The score at a decision node $(v_t,\mathcal{W}_t)$ depends on the visit count of the node, $N(v_t,\mathcal{W}_t)$, and the visit count of each action $a \in A$, $N(v_t,a,\mathcal{W}_t)$.
An exploration hyperparameter $c_e$ controls the exploration–exploitation trade-off: larger $c_e$ encourages more exploration (potentially better solutions) at the cost of additional compute.
Eq.~(5) defines the tree policy $\pi_{tree}$ at decision nodes.
\begin{equation*}
a^\star \;\in\; \arg\min_{a\in A}
\Big(Q(v_t,a,\mathcal{W}_t)\;-\;c_{e}\sqrt{\tfrac{\ln N(v_t,\mathcal{W}_t)}{N(v_t,a,\mathcal{W}_t)}} \Big).
\end{equation*}
At chance nodes, the tree policy \(\pi_{\mathrm{tree}}\) is probabilistic.
Let the two child decision nodes have belief subsets \(\mathcal{W}^{+}\) and \(\mathcal{W}^{-}\) (with parent set \(\mathcal{W}\)).
We select the child \((v, \mathcal{W}^{+})\) with probability:
\begin{equation*}
\begin{aligned}
\pi_{\mathrm{tree}}\!\big((v, \mathcal{W}^{+}) \mid (v,a; \mathcal{W})\big) &= \frac{|\mathcal{W}^{+}|}{|\mathcal{W}|},\\
\pi_{\mathrm{tree}}\!\big((v,\mathcal{W}^{-}) \mid (v,a;\mathcal{W})\big) &= \frac{|\mathcal{W}^{-}|}{|\mathcal{W}|}.
\end{aligned}
\end{equation*}

This cardinality-based sampling biases visits toward children with larger belief set $\mathcal{W}$, which typically require more iterations to expand due to a larger number of potential leaf nodes.

\smallskip
\noindent
\textbf{\emph{(2) Expansion.}}
Upon reaching incomplete node $(v,\mathcal{W})$, we pick at random one action 
$a\in A$ such that $(v,a,\mathcal{W})$ does not yet exist in the tree.
We add it to the tree, as well as its two children 
$(v',\mathcal{W}^+)$ and $(v,\mathcal{W}^-)$.
When adding a new node, we ensure that the same configuration \((v, \mathcal{W})\) does not already appear along the ancestor path (cycle avoidance). If this is the case, then this action is known not to be optimal, and the associated cost is fixed to $+\infty$.
Both chance nodes $(v',\mathcal{W}^+)$ and $(v,\mathcal{W}^-)$ are evaluated using {\em rollouts}.

\smallskip
\noindent
\textbf{\emph{(3) Rollout.}}
We compute the evaluation $Q(v',\mathcal{W}^+)$ and $Q(v,\mathcal{W}^-)$ of the cost of 
$(v',\mathcal{W}^+)$ and $(v,\mathcal{W}^-)$, based on Eq.~(4).

Designing effective rollouts in BWTS is crucial as rollout returns directly affect the current value estimates and, consequently, future selections under the tree policy $\pi_{tree}$.
Naive (random) rollouts, as commonly used in MCTS, can be uninformative in some domains (such as mazes) where loops and dead-ends induce high variance and may bias exploration toward undesirable, suboptimal policies.
To reduce variance and stabilize backups, we propose \emph{strategic rollouts} instead of purely random ones.

At each AND node, we evaluate a fixed set of ten rollout strategies $\{\sigma_{k}\}_{k=1}^{10}$.
For a given strategy $\sigma_k$ and each belief set $W_i \in \mathcal{W}_t$ present at the current node, we execute the same strategy policy and obtain a return cost $c^{k}(W_i) = C((v_t, \dots, g); \sigma_k, W_i)$.
We then compute the probability-weighted average
$$c^k_{avg} =\sum_i p(W_i)\, c^{k}(W_i),$$
and use the minimum across strategies, $c_{roll} = \min_{k \in \{1,\dots,10\}} c^{k}_{avg}$, as the rollout estimate for that node.

We now define each rollout strategy $\sigma_k$.
Each strategy has two components:
(i) a primary target has one of five choices: move toward the midpoint of each border (four choices) or move toward the goal; and
(ii) a sweep rule when encountering a blocking wall: perform a directed sweep along the wall to find an opening, with two variants, left-priority vs.\ right-priority (for vertical walls, these correspond to up-first vs.\ down-first; for horizontal walls, left-first vs.\ right-first).
Combining the $5$ primary targets with the $2$ sweep variants yields $5 \times 2 = 10$ rollout strategies in total. These strategic rollouts lower variance, stabilize AND/OR backups, and reduce the number of iterations required to grow a useful search tree.



\smallskip
\noindent
\textbf{\emph{(4) Backpropagation.}}
After computing the evaluations $Q(v',\mathcal{W}^+)$ and $Q(v,\mathcal{W}^-)$, 
we backtrack inductively from the bottom to the root and update the values of ancestor nodes.
For chance nodes $(v, a, \mathcal{W})$ with children 
\((v, \mathcal{W}^-), (v', \mathcal{W}^+)\), 
we set as in Eq.~(2):
\begin{equation*}
 \label{eq:chance_value}
Q(v,a,\mathcal{W})= \frac{|\mathcal{W}_1| Q(v_1,\mathcal{W}_1) + |\mathcal{W}_2|Q(v_2,\mathcal{W}_2)}{|\mathcal{W}_1|+|\mathcal{W}_2| = |\mathcal{W}|}    
\end{equation*}

For decision nodes \((v, \mathcal{W})\),
we set as in Eq.~(3):
\begin{equation*}
\label{eq:decision_value}
Q(v,\mathcal{W}) \;\leftarrow\; \min_{a \in A} (C(v, a) \;+\; Q(v,a,\mathcal{W}))
\end{equation*}


These four stages are repeated for a given iteration budget.

\section{BAMCP Discussion}

BAMCP is a symbolic planner used at the high level in neurosymbolic HRL frameworks and also allows the incorporation of prior knowledge. It provides an approximate solution to the Bayesian Adaptive Markov Decision Process (BAMDP), which formulates the MDP objective over all possible transition probability distributions. Solving the BAMDP exactly is intractable, as it requires integrating over all such distributions.
BAMCP addresses this by assuming independence across state–action pairs: for each pair ($s,a$), it models a distribution over next states ($s'$) independently of other pairs. This is typically implemented using a Dirichlet prior over transition probabilities, where the posterior is updated only from observed tuples ($s,a,s'$). 
Consequently, BAMCP captures prior knowledge solely at the level of one-step transition dynamics. However, in the Belief World Problem, knowledge is structural, hence it creates dependencies across positions. For instance, we know there is a (horizontal or vertical) wall, but we do not know where. This creates dependencies e.g. between cells $(3,3)$ and $(7,7)$ that cannot have a wall at the same time. Such dependencies cannot be represented under independent transition model of BAMCP. The core novelty of BWTS is to explicitly model and reason over structural knowledge, and handle dependencies between positions.

Still, we attempted to incorporate structural knowledge into BAMCP and conduct various experiments, reported in Table~\ref{tab:bamcp}. A naive prior is to assume that, for every unexplored cell, there is no wall with probability one (so the agent moves with $p_{\text{move}}=1$) and a wall with probability zero (so $p_{\text{stay}}=0$). Once a wall is detected in position $X$, these probabilities are reversed for position $X$, i.e., $p_{\text{move}}=0$ and $p_{\text{stay}}=1$. Indeed, this is exactly the assumption made by $D^*$, which plans under the belief that all cells are initially free and replans in the next cycle using the updated knowledge after wall detection. We also experimented with a prior $p_{\text{move}}=0.9$ and $p_{\text{stay}}=0.1$ to reflect the proportion of walls from $\mathcal{W}_{h},\mathcal{W}_{v},\mathcal{W}_{hv}$.

BAMCP performs poorly under both priors, incurring roughly twice the cost of $D^*$ (Table 1). The reason
is 
tied to the use of stochastic rollouts, which guide the agent’s exploration. 
With prior 
($p_{\text{move}}=0.9$, $p_{\text{stay}}=0.1$)
or
($p_{\text{move}}=1$, $p_{\text{stay}}=0$),
each rollout generates a world where $10\%$ of the cells is a wall on average. This allows a lot of possible trajectories for the rollouts, including a lot of looping, as illustrated on Fig. \ref{fig:bamcp_rollout}, and very few rollouts reach the goal (both rollout 1 and 2 fail to reach the goal after 120 steps). This is a key reason why we use strategic rollouts in BWTS.

\begin{table}[t]
  \centering
  \caption{Expected-cost (Total run time for solving all grid configuration), lower is better, for the BWTS-learned policy, BAMCP ($p_{move}/p_{stay}$), and $D^*$ across ($\mathcal{W}_h$), ($\mathcal{W}_v$), and  ($\mathcal{W}_{hv}$).}
  \label{tab:bamcp}
  \resizebox{0.9\columnwidth}{!}{
  \begin{tabular}{@{}lccc@{}}
    \toprule
    Belief Worlds & $\mathcal{W}_h$ & $\mathcal{W}_v$ & $\mathcal{W}_{hv}$ \\
    \midrule
    
    \multirow{2}{*}{$D^*$}
      & $24.56$ & $23.19$ & $23.88$ \\
      & $\mathbf{(0.03s)}$ & $\mathbf{(0.03s)}$ & $\mathbf{(0.06s)}$ \\


    \multirow{2}{*}{BWTS}
      & $\mathbf{21.22 \pm 0.20}$ & $\mathbf{21.22 \pm 0.16}$ & $\mathbf{23.36 \pm 0.08}$ \\
      & $(84.51s \pm 12.24)$             & $(84.28s \pm 9.58)$             & $(834.24s \pm 196.96)$ \\

    \midrule
      \multirow{2}{*}{BAMCP (1.0/0.0)}
      & $49.61 \pm 1.49$ & $54.76 \pm 5.58$ & $50.81 \pm 1.59$ \\
      & $(4313.19s \pm 116.40)$    & $(4814.18s \pm 378.05)$            & $(8733.22s \pm 230.11)$ \\

     \cmidrule(lr){2-4}

      \multirow{2}{*}{BAMCP (0.9/0.1)}
      & $45.85 \pm 2.69$ & $50.55 \pm 2.10$ & $50.70 \pm 2.43$ \\
      & $(4215.38s \pm 348.65)$    & $(4667.55s \pm 225.59)$            & $(9089.90s \pm 522.40)$ \\

     \cmidrule(lr){2-4}

    \multirow{2}{*}{BAMCP (0.6/0.4)}
      & $\mathbf{28.42 \pm 1.16}$ & $\mathbf{28.89 \pm 1.29}$ & $\mathbf{40.42 \pm 1.39}$ \\
      & $(2469.89s \pm 132.08)$    & $(2312.52s \pm 83.85)$            & $(7367.34s \pm 351.41)$ \\

    \bottomrule
  \end{tabular}
  }
\end{table}

\begin{figure}[t] 
  \centering
  
   \includegraphics[width=0.3\linewidth]{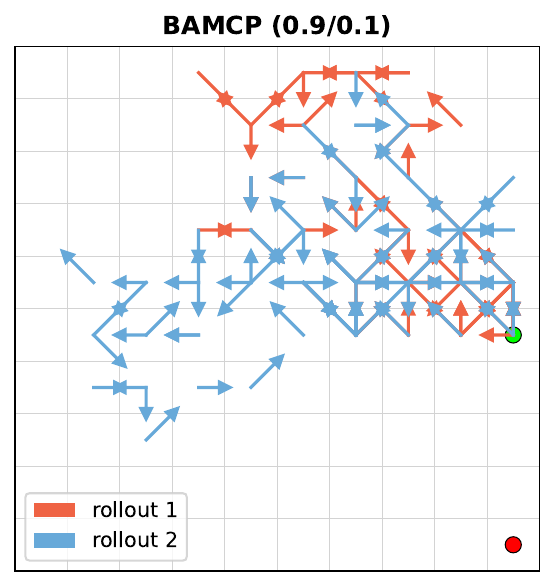}
  \caption{Samples of random rollout trajectories of length around $120$, starting from the \textcolor{green}{green dot} at $(5,9)$ under BAMCP $(0.9/0.1)$, fail to reach the goal \textcolor{red}{red dot} at $(9,9)$.
}
  \label{fig:bamcp_rollout}
\end{figure}

\begin{figure}[b] 
  \centering
  \begin{subfigure}{0.18\linewidth}
    \centering
    \includegraphics[width=\linewidth]{figs/two_rooms/bamcp_multi_27_52.pdf}
    \caption{Trajectory}
    \label{fig:traj_bamcp0.6}
  \end{subfigure}
  \begin{subfigure}{0.5\linewidth}
    \centering
    \includegraphics[width=0.82\linewidth]{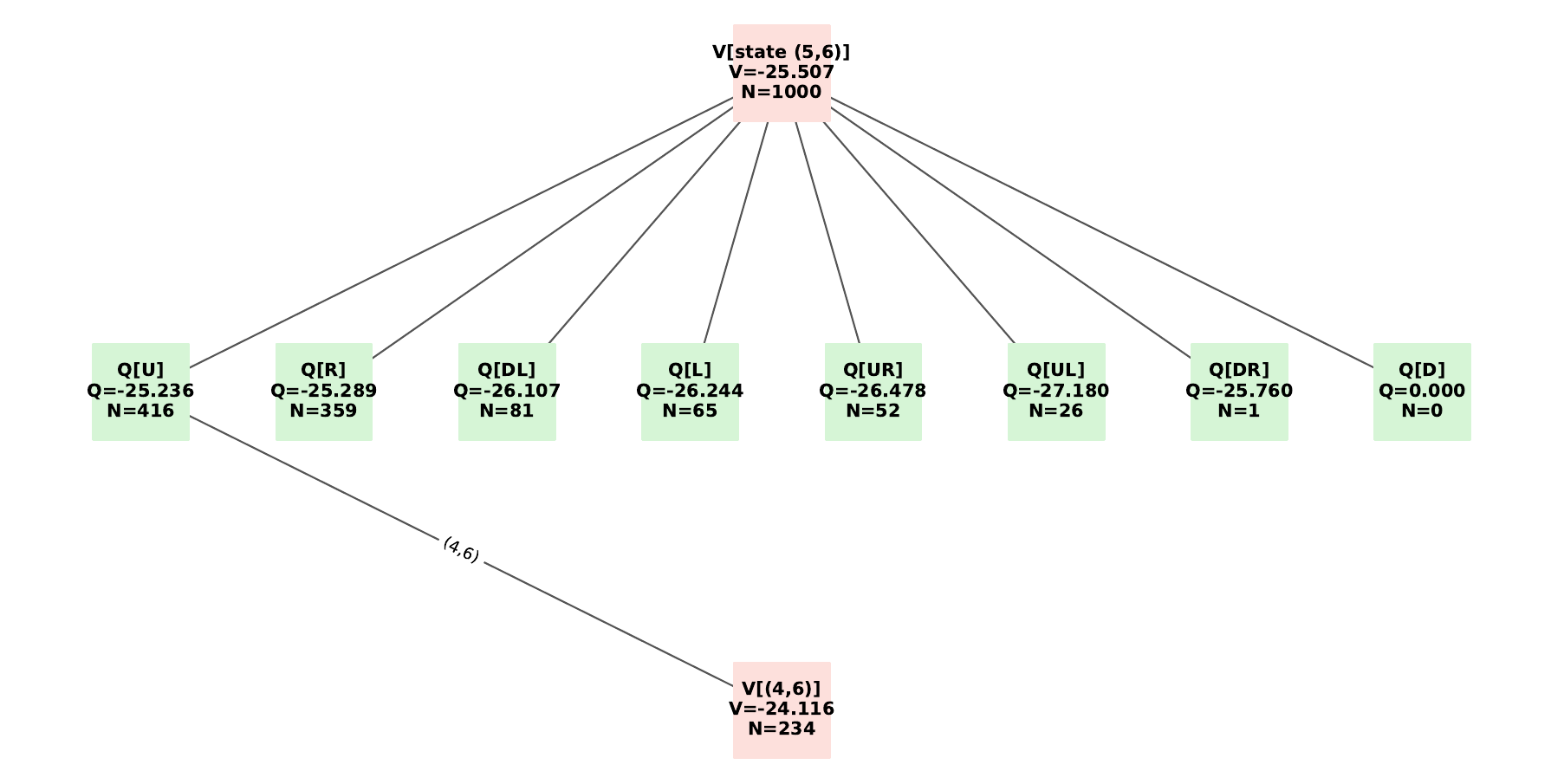}
    \caption{BAMCP root node at location (5,6)}
    \label{fig:bamcp_tree}
  \end{subfigure}
  \caption{The agent oscillates between moving up at $(5,6)$ and moving down again for BAMCP (0.6/0.4) due to difficulties due to stochastic setting of $p_{move}=0.6$ and $p_{stay}=0.4$.}
  \label{fig:bamcp_loop}
\end{figure}

The best-performing prior we found was setting $p_{\text{move}}=0.6$ and $p_{\text{stay}}=0.4$, which is closer to $D^*$, $20\%$ less sample efficient than $D^*$. With $40\%$ of walls, the space is more constrained, there are fewer loops, and more stochastic rollouts find the goal.
The execution time is substantially higher than $D^*$, and even higher than BWTS, which computes the optimal policy, whereas BAMCP is even less efficient than $D^*$: for structural knowledge, BAMCP is just not adequate.
We finally illustrate that
 the strategy found by BAMCP is also random and does not follow any pattern like in $D^*$ and BWTS (see main paper). As shown in Fig. \ref{fig:bamcp_loop}, the agent can oscillate between two cells incurring unnecessary cost, which is avoided by the 
 more strategic methods $D^*$ and BWTS.

\section{Experiment Details and Additional Results}

The maps for the Point Maze environments used in our experiments are illustrated in Fig. \ref{fig:point_maze_maps}. Each environment consists of a continuous two-dimensional navigation domain, requiring the agent to reach a goal from a start point. The low-level policy is trained following the same procedure as in RGL (using SAC), where a continuous control policy is optimized to navigate between locally specified subgoals in an obstacle-free environment. The low-level training uses a curriculum reward shaping in which the goal distance for each episode is randomized and gradually increased over training episodes. The resulting training reward curve is shown in Fig. \ref{fig:reward_curve}, demonstrating stable convergence of the low-level controller.


\begin{figure}[h] 
  \centering
  \begin{subfigure}{0.25\linewidth}
    \centering
    \includegraphics[width=\linewidth]{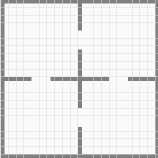}
    \caption{Four Rooms}
    \label{fig:four_rooms_map}
  \end{subfigure}
  \begin{subfigure}{0.25\linewidth}
    \centering
    \includegraphics[width=\linewidth]{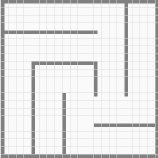}
    \caption{Medium Maze}
    \label{fig:medium_maze_map}
  \end{subfigure}
  \begin{subfigure}{0.25\linewidth}
    \centering
    \includegraphics[width=1.0\linewidth]{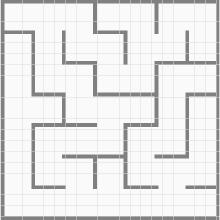}
    \caption{Hard Maze}
    \label{fig:hard_maze_map}
  \end{subfigure}
  \caption{Point Maze Environment Maps}
  \label{fig:point_maze_maps}
\end{figure}


\begin{figure}[h]
  \centering
  \includegraphics[width=0.3\linewidth]{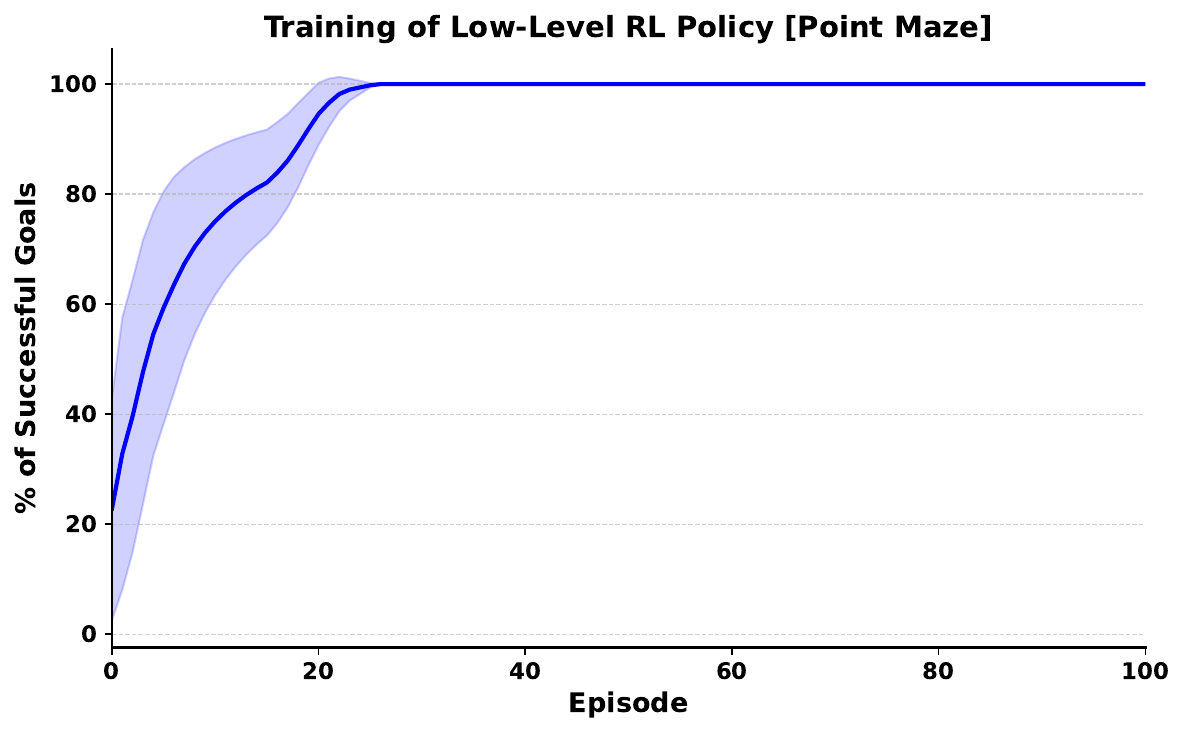}
  \includegraphics[width=0.3\linewidth]{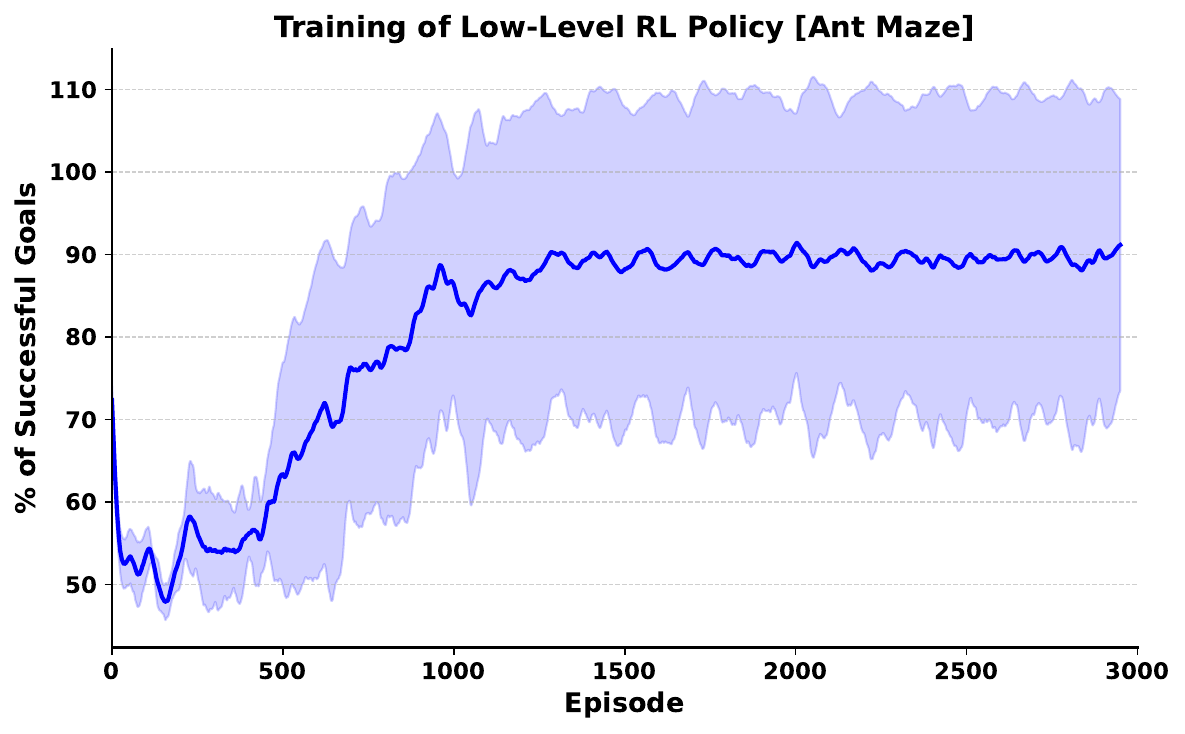}
  \caption{Low-level policy training curve of point maze agent. Percentage of successful goals in each episode. The environment is sparse reward and gets 1 after reaching the goal, 0 otherwise.}
  \label{fig:reward_curve}
\end{figure}

\smallskip
\noindent
\textbf{InK in PointMaze FourRooms}:
We illustrate this on Fig.~\ref{fig:point_agent}
with one average case on the 4 room environment. 

\begin{figure}[t!] 
  \centering
  \begin{subfigure}{0.2\linewidth}
    \centering
    \includegraphics[width=\linewidth]{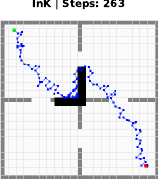}
    \caption{}
    \label{fig:pm_bwts_fourrooms}
  \end{subfigure}
  \begin{subfigure}{0.2\linewidth}
    \centering
    \includegraphics[width=1.0\linewidth]{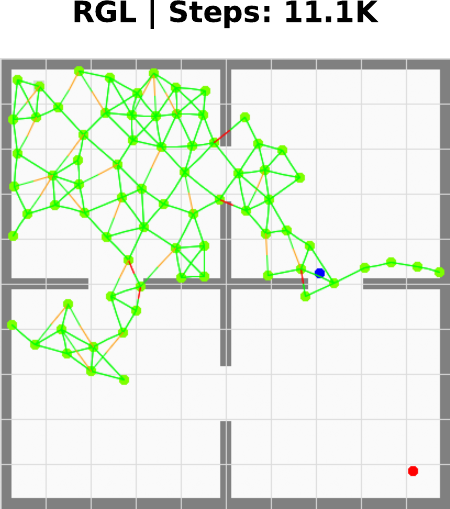}
    \caption{}
    \label{fig:pm_rgl_fourrooms}
  \end{subfigure}
  \caption{Runs of InK ($\emptyset$) vs RGL training in 4 rooms.}
  \label{fig:point_agent}
\end{figure}

\smallskip
\noindent
\textbf{InK in Ant-Maze U-Room}: We demonstrate in Fig.~\ref{fig:ant_maze} that InK can handle even more challenging environments, such as the Ant-Maze U-Room. Here, the difference lies only in the low-level policy, while high-level planning is performed using InK symbolic planners.

\begin{figure}[h!]
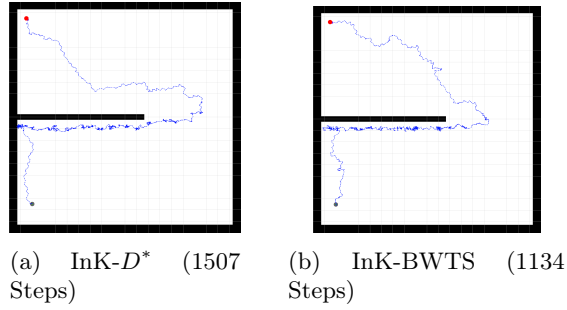
 
  \centering
  \begin{subfigure}{0.25\linewidth}
    \centering
    \includegraphics[width=\linewidth,trim={0 0 0 4mm},clip]{figs/ant_maze/Ant_Dstar_urooms_0.pdf}
    \caption{InK-$D^*$ (1507 Steps)}
    \label{fig:ant_maze_dstar}
  \end{subfigure}
  \quad \; 
  \begin{subfigure}{0.3\linewidth}
    \centering
    \includegraphics[width=0.82\linewidth,trim={0 0 0 4mm},clip]{figs/ant_maze/Ant_bwts_urooms_0.pdf}
    \caption{InK-BWTS (1134 Steps)}
    \label{fig:ant_maze_bwts}
  \end{subfigure}
  \caption{Ant-Maze U-Room RL Environment.}
  \label{fig:ant_maze}
\end{figure}


\section{InK-BWTS for RL Environments}

We provide here more explanation for line 569-571,
about when BWTS is used without prior knowledge as the symbolic planner. Then, the set of all mazes is too large to be used as belief set $\mathcal{W}_0$.
So we use a subset of $\mathcal{W}_{hv}$ as a representative set of the next wall to be encountered.

InK alternates between high-level planning and low-level execution while incrementally updating knowledge about the environment which is shown in Algorithm 2. 

When BWTS is used as the symbolic planner, 
planning is restricted to a relevant subregion determined by a provisional path computed under current knowledge (line 3 and 4), different at each iteration. Within this region, a belief set is generated using simple wall primitives (horizontal and vertical), reflecting the intuition that complex maze layouts are composed of elementary horizontal or vertical walls with openings (line 5). High-level symbolic planning is then performed using BWTS (Algorithm 1) to select an intermediate subgoal (line 6), which is executed by a low-level controller with monitoring $mon$ (line 8). Newly discovered obstacles update the knowledge map $M$, and the process repeats until the goal is reached.

\begin{algorithm}[]
\caption{InK-BWTS with a representative set $\mathcal{W}_{bb}$}
\label{alg:ink_bwts}
\begin{algorithmic}[1]
\Require Current abstract knowledge map $M$, start state $s_0$, goal state $g$, low-level policy $\pi_\ell$
\Ensure Agent reaches goal $g$

\State Initialize current state $s \leftarrow s_0$
\While{$s \neq g$}
    \State Run A* on current knowledge $M$ from $s$ to $g$ to obtain provisional path $P$
    \State Compute the smallest axis-aligned bounding box $B$ enclosing $P$
    \State $\mathcal{W}_{bb} \leftarrow \textsc{GenerateBeliefSet}(M, B)$
    \State Run BWTS$(\mathcal{W}_{bb}, s, g)$ to construct BWTS tree 
    \State Extract next high-level subgoal $\hat{s}$ from the root of BWTS tree
    \State Execute low-level policy $\pi_l$ from $s$ toward $\hat{s}$ with collision monitoring ($mon$)
    \If{unexpected wall encountered during execution}
        \State Update knowledge map $M$ with detected obstacle
    \Else
        \State Update current state $s \leftarrow \hat{s}$
    \EndIf
\EndWhile
\State \Return success
\end{algorithmic}
\end{algorithm}

\begin{figure}[] 
  \centering
  \begin{subfigure}{0.2\linewidth}
    \centering
    \includegraphics[width=0.7\linewidth]{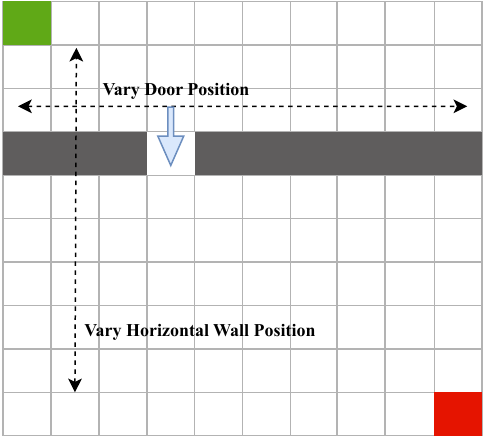}
    \caption{Horizontal Walls}
    \label{fig:bw_hor}
  \end{subfigure} 
  \begin{subfigure}{0.2\linewidth}
    \centering
    \includegraphics[width=0.7\linewidth]{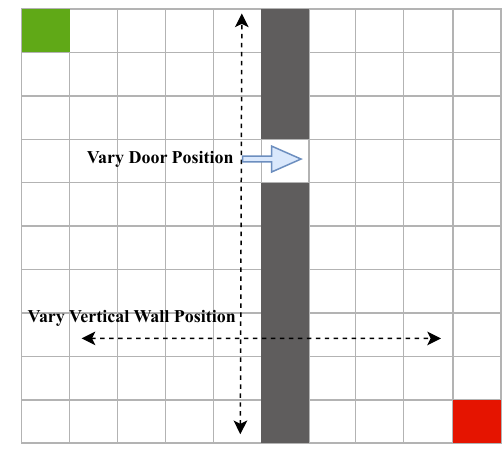}
    \caption{Vertical Walls}
    \label{fig:bw_ver}
  \end{subfigure}
  \caption{Belief Worlds Generation for Maze Environment.}
  \label{fig:bw_generation}
\end{figure}


\subsection{Belief Set Generation for InK-BWTS}

The belief set construction is designed to encode a structural prior over maze layouts while remaining computationally tractable which is provided in Algorithm 3.
The set of belief worlds is composed of environments containing a \emph{single horizontal} or a \emph{single vertical} wall, each with exactly one opening. This design reflects the intuition that complex maze structures can be decomposed into compositions of such primitive wall elements. As illustrated in Fig.~\ref{fig:bw_generation}, for a $10\times10$ grid we generate horizontal-wall belief worlds by placing a horizontal wall at each admissible row and inserting a single opening. The start row ($y=0$) and goal row ($y=9$) are excluded from wall placement to ensure feasibility of trivial paths. This yields $8$ admissible wall rows; for each row, the opening may be placed in any of the $10$ columns, giving rise to $|\mathcal{W}_h| = 80$ horizontal-wall belief worlds (Fig.~\ref{fig:bw_hor}).

\begin{algorithm}[]
\caption{GenerateBeliefSet $\mathcal{W}_{bb}$}
\label{alg:generate_belief}
\begin{algorithmic}[1]
\Require Current knowledge map $M$, bounding box $B$
\Ensure Belief-world set $\mathcal{W}_{bb}$

\State Initialize $\mathcal{W}_{bb} \leftarrow \emptyset$
\State Restrict knowledge map $M$ to region inside $B$ to obtain $M_B$
\State Add $M_B$ to $\mathcal{W}_{bb}$

\ForAll{wall orientation $o \in \{\text{horizontal}, \text{vertical}\}$}
    \ForAll{wall offsets $d \in \{2, \dots,8\}$ within $B$} \Comment{for reduced set offset is $\{1,2,5,9\}$}
        \State Place a wall of orientation $o$ at offset $d$
        \ForAll{valid opening positions along the wall}
            \State Construct belief world $W$ by inserting the wall with a single opening into $M_B$
            \State Add $W$ to $\mathcal{W}_{bb}$
        \EndFor
    \EndFor
\EndFor

\State \Return $\mathcal{W}_{bb}$
\end{algorithmic}
\end{algorithm}

Vertical-wall belief sets are constructed analogously by placing a vertical wall at each admissible column except the start column ($x=0$) and goal column ($x=9$), again with a single opening per wall. This produces $|\mathcal{W}_v| = 80$ vertical-wall belief worlds (Fig.~\ref{fig:bw_ver}). We further define the combined belief set as $\mathcal{W}_{hv} = \mathcal{W}_h \cup \mathcal{W}_v$.

In the full InK-BWTS pipeline, belief sets are not generated over the entire grid. Instead, they are constructed only within a dynamically computed bounding box ($B$) that encloses the provisional $A^*$ path. That is, the planner crops the relevant section of the map for belief-world construction rather than operating over the full grid. This design is particularly beneficial for larger grid sizes, as it significantly accelerates BWTS construction.

Additionally, we may not want to place walls at every possible horizontal or vertical position. Instead, walls are generated only at selected positions to construct the belief set, as done in line~5 of Algorithm~\ref{alg:generate_belief}, where the wall location is chosen by selecting an offset from the agent toward the goal. 
To further restrict the belief space, we reduce the number of possible wall openings for walls that are far from the agent. For example, in a $10\times10$ grid with a single wall at offset 9, instead of considering all 10 possible configurations corresponding to different opening locations, we consider only alternate positions, reducing the number of openings to 5.

\end{document}